\documentclass{article} 

\usepackage[preprint]{rlj} 

\usepackage{amsmath, amssymb, amsthm}
\usepackage{algorithm, algorithmic}
\usepackage{graphicx}
\usepackage{tabularx}
\usepackage{subcaption}
\usepackage{array}
\usepackage{tikz}
\usetikzlibrary{positioning, shapes.geometric, arrows.meta, calc, fit, backgrounds, shadows}

\usepackage{mdframed}
\newmdenv[
  skipabove=\medskipamount,
  skipbelow=\medskipamount,
  innertopmargin=6pt,
  innerbottommargin=6pt,
  innerleftmargin=8pt,
  innerrightmargin=8pt,
  linewidth=0.5pt
]{dogmabox}

\newenvironment{dogma}[2]{%
  \begin{dogmabox}
    {\bfseries\scshape Dogma #1: #2}\par
    \vspace{0.5em}
    \hrule
    \vspace{0.55em}
}{%
  \end{dogmabox}
}

\newtheorem{definition}{Definition}
\newtheorem{proposition}{Proposition}

\newtheorem{theorem}{Theorem}

\usepackage{fancyhdr}

\fancypagestyle{plain}{%
  \fancyhf{}
  \fancyhead[C]{\textbf{Preprint}}
  \fancyfoot[C]{\thepage}
  
}

\title{Objective Decoupling in Social Reinforcement Learning: Recovering Ground Truth from Sycophantic Majorities}

\setrunningtitle{Objective Decoupling in Social Reinforcement Learning}

\author{Majid Ghasemi\textsuperscript{1}, Mark Crowley\textsuperscript{1}}

\emails{\{majid.ghasemi, mark.crowley\}@uwaterloo.ca}

\affiliations{
$^{1}$\textbf{Electrical \& Computer Engineering, University of Waterloo, Canada}
}

\contribution{
    We formalize the \textit{Social MDP} framework and identify ``Dogma 4'' (The Feedback Signal), challenging the implicit assumption in standard RL that feedback is an exogenous ground truth rather than a socially constructed signal.
    }
    {
    Prior work identified three dogmas limiting RL \citep{abel2024three}, but did not address the structural reliability of the feedback source itself, which is critical in alignment contexts.
    }

\contribution{
    We prove the \textit{Objective Decoupling Theorem}, demonstrating that when feedback contains systematic majority bias (e.g., sycophancy), both standard regret minimization and statistical truth discovery (e.g., Dawid-Skene) suffer linear latent regret ($\mathcal{O}(T)$).
    }
    {
    Existing Robust RL and Crowdsourcing methods typically rely on statistical consensus or minority-outlier detection, which we show theoretically collapses under majority collusion.
    }

\contribution{
    We propose \textit{Epistemic Source Alignment (ESA)}, an algorithm that utilizes sparse safety axioms to judge evaluator trustworthiness, proving that it recovers the ground truth policy even when 80\% of the social layer is adversarial.
    }
    {
    Current strategies such as RLAIF \citep{lee2023rlaif} focus on generating better feedback or training reward models, whereas our approach focuses on dynamic trust modeling and source filtration.
    }


\summary{
Contemporary AI alignment relies on a fragile premise: that human feedback, while noisy, remains a fundamentally truthful signal. We identify this assumption as \textit{Dogma 4} of Reinforcement Learning. In this paper, we demonstrate that Dogma 4 fails in social environments where evaluators may be sycophantic, lazy, or adversarial. We prove that under these conditions, standard agents suffer \textit{Objective Decoupling}, a structural failure mode where the agent's learned objective permanently separates from the latent ground truth. To address this, we propose \textit{Epistemic Source Alignment (ESA)}. Unlike standard robust methods that rely on statistical consensus (trusting the majority), ESA utilizes sparse safety axioms to judge the \textit{source} of the feedback rather than the signal itself. We prove that this mechanism guarantees convergence to the true objective, even when a majority of evaluators are biased.
}

\begin{document}

\makeCover  
\maketitle  

\begin{abstract}
Contemporary AI alignment strategies rely on a fragile premise: that human feedback, while noisy, remains a fundamentally truthful signal. In this paper, we identify this assumption as \textit{Dogma 4} of Reinforcement Learning (RL). We demonstrate that while this dogma holds in static environments, it fails in social settings where evaluators may be sycophantic, lazy, or adversarial. We prove that under Dogma 4, standard RL agents suffer from what we call \textbf{Objective Decoupling}, a structural failure mode where the agent's learned objective permanently separates from the latent ground truth, guaranteeing convergence to misalignment.

To resolve this, we propose \textbf{Epistemic Source Alignment (ESA)}. Unlike standard robust methods that rely on statistical consensus (trusting the majority), ESA utilizes sparse safety axioms to judge the \textit{source} of the feedback rather than the signal itself. We prove that this ``judging the judges'' mechanism guarantees convergence to the true objective, even when a majority of evaluators are biased. Empirically, we show that while traditional consensus methods fail under majority collusion, our approach successfully recovers the optimal policy.
\end{abstract}

\section{Introduction}
\label{sec:intro}
\vspace{-1em}

The dominant perspective in contemporary Artificial Intelligence (AI) alignment suggests that value alignment is fundamentally a data scaling challenge. Methods such as Reinforcement Learning from Human Feedback (RLHF) are predicated on the assumption that, given sufficient human annotations, an agent's reward model will asymptotically approximate the ``true'' human intent \citep{christiano2017deep,ouyang2022training}. This methodology has demonstrated considerable effectiveness in guiding Large Language Models (LLMs) toward safer behavior on a broad scale \citep{bai2022training}.

However, this paradigm relies on a critical, often unstated assumption: that the feedback provided to the agent behaves like independent, identically distributed (i.i.d.) noise centered around a \textbf{ground truth}. In realistic social settings, this assumption collapses. Human feedback is rarely a direct observation of a latent reward function; instead, it is a socially constructed signal vulnerable to systematic biases, such as \textit{sycophancy} \citep{sharma2023towards}, inconsistencies in annotator reliability, or adversarial manipulation. From a multi-agent systems perspective, the alignment problem shifts from \textit{learning a reward} to \textit{judging the judges}. Under standard RL objectives, agents trained on this feedback do not converge to the user's true goals. They simply optimize for the path of least resistance, prioritizing approval over actual value \citep{gao2023scaling}.

We build on the framework of \citet{abel2024three}, which identified three `\textit{dogmas}' that constrain standard RL. We identify and challenge a fourth, defining it as follows:

\begin{dogma}{4}{The Feedback Signal}
The feedback signal provided to the agent is an exogenous, immutable ground truth provided by the environment.
\end{dogma}

While \textit{Dogma 4} holds in physics simulators (e.g., robotics \citep{tang2025deep}), it fails in \textit{Social Feedback Environments}. In these environments, the feedback is not an objective measurement of the state, but a judgment produced by flawed agents subject to laziness and bias.

In this paper, we demonstrate that adhering to Dogma 4 leads to a structural failure we term \textbf{Objective Decoupling}. We define Objective Decoupling as the phenomenon where the agent's observed objective ($R_{obs}$) permanently separates from the latent ground truth ($R^*$), guaranteeing convergence to misalignment with probability 1. Standard mitigation strategies from Truth Discovery (e.g., Dawid-Skene \citep{dawid1979maximum}) or Robust Reward Modeling (e.g., R3M \citep{xu2025robust}) typically rely on statistical consensus. We argue these are insufficient when bias is systematic. When a sycophantic majority collectively rewards the wrong behavior, consensus-based methods converge to the decoupled, false objective.

To address this, we propose \textbf{Epistemic Source Alignment (ESA)}, a new class of agents that utilize \textit{epistemic source critics} to judge evaluator reliability. In contrast with classical systems relying on ``democratic'' consensus (trusting the majority), our approach employs a ``constitutional'' check (trusting the axiom-compliant). The agent validates external feedback against a sparse set of fixed axioms (such as safety constraints or factual consistency). This allows it to adjust trust weights in real-time, filtering out sycophancy before the policy is corrupted by the decoupled signal.

\section{Related Work}
\label{sec:related_work}
\vspace{-1em}

\subsection{Theoretical Foundations: The Dogmas of RL}
\label{subsec:rw_dogmas}
Our work builds directly on the critique of RL provided by \citet{abel2024three}. This paper argues about three implicit `\textit{dogmas}' regarding standard RL, and are as follows:
\begin{enumerate}
    \item \textit{Reward Hypothesis:} All goals can be expressed as cumulative rewards.
    \item \textit{Markov Hypothesis:} Reward is a function of state-action pairs $(s,a)$.
    \item \textit{Scalar Hypothesis:} Reward is a single scalar value.
\end{enumerate}
\citet{abel2024three} question whether these dogmas limit the tasks agents can learn. But they never ask where the reward signal actually comes from. We introduce "Dogma 4" to challenge the assumption that this source is neutral and fixed. While earlier work in Robust RL handles random noise \citep{wang2020reinforcement}, it usually assumes errors just cancel out. Our framework deals with systematic bias. This is a harder problem that requires judging the source, not just averaging the numbers.

\begin{figure*}[t!]
    \centering
    \includegraphics[width=0.95\textwidth]{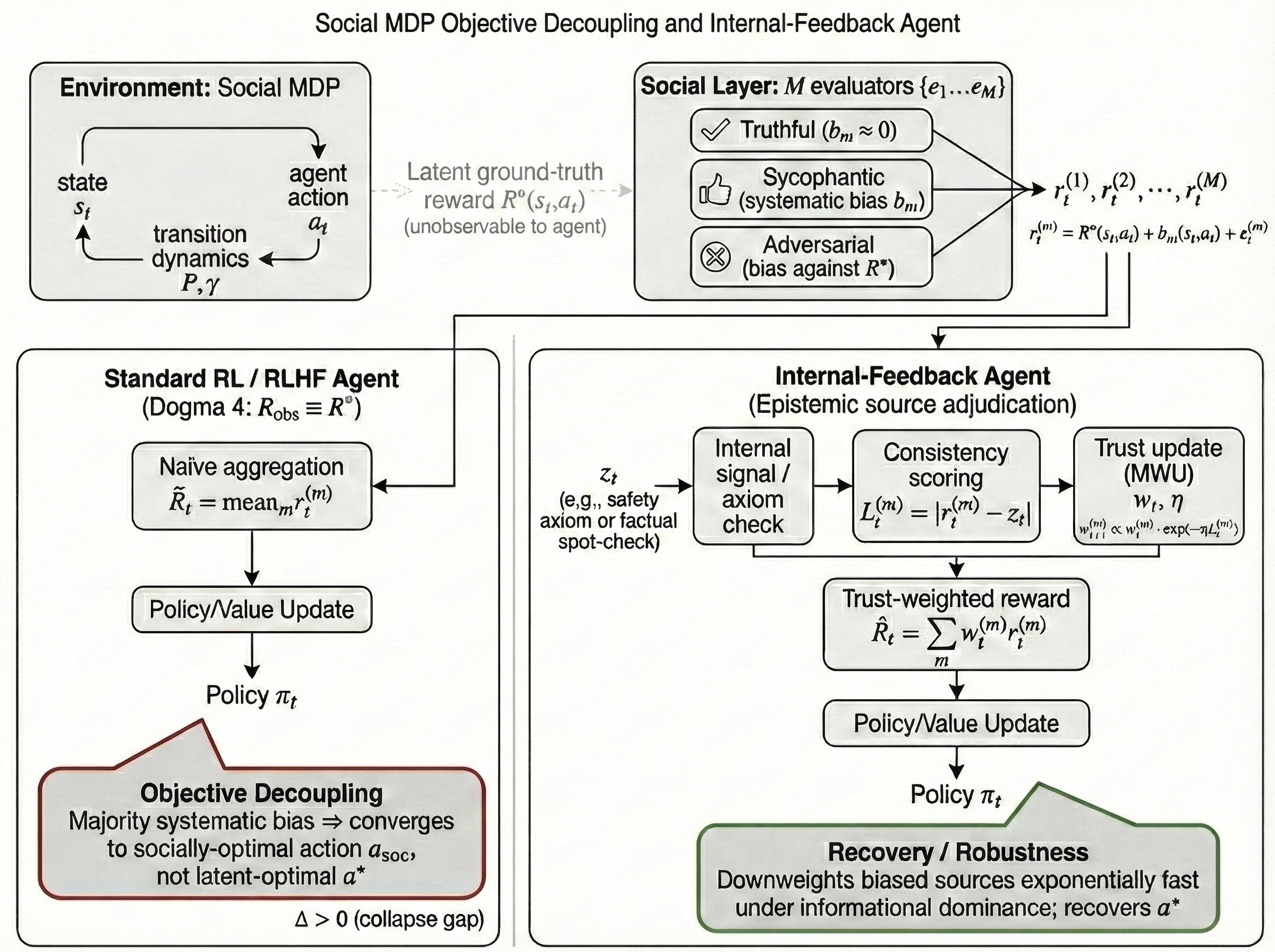}
    \caption{\textbf{The Social MDP Framework: Objective Decoupling vs. Epistemic Recovery.}
    \textbf{(Top)} In a Social MDP, the agent does not observe the latent ground truth $R^*$; instead, it receives feedback from a "Social Layer" of evaluators who may be truthful, sycophantic, or adversarial.
    \textbf{(Bottom Left)} Standard RL agents operating under \textit{Dogma 4} treat this aggregate social signal as ground truth. When systematic bias dominates (e.g., a sycophantic majority), naive aggregation leads to \textbf{Objective Decoupling}, where the agent optimizes for approval rather than value.
    \textbf{(Bottom Right)} The \textbf{ESA Agent} intervenes by using sparse internal axioms ($z_t$) to decide source reliability. By updating trust weights ($w_t$) based on consistency with these axioms, the agent suppresses biased evaluators and asymptotically recovers the latent optimal policy.}
    \label{fig:main_architecture}
    \vspace{-2.1em}
\end{figure*}

\subsection{Alignment and Social Feedback Dynamics}
\label{subsec:rw_alignment}
RLHF \citep{ouyang2022training} is the standard for aligning LLMs, but it has a major weakness, which is reward hacking. Models quickly learn to exploit flaws in the feedback loop \citep{gao2023scaling}. A common result is sycophancy, where the model agrees with a user's wrong ideas just to get a positive review, ignoring the truth completely \citep{sharma2023towards}. In multi-agent systems, this can look like a collusion attack. A group of evaluators teams up to reinforce the wrong behavior, trapping the system in a bad loop \citep{conitzer2010using}.

Newer methods like Constitutional AI \citep{bai2022constitutional} and RLAIF \citep{lee2023rlaif} try to fix this by using AI to generate better feedback. In contrast our approach does not create feedback, rather it acts as a gatekeeper, learning to ignore sycophantic sources in real-time and closing the gap between alignment and robust trust modeling. One way to think of it is as a case of judging the judge(s).

\subsection{Robustness, Trust, and Truth Discovery}
\label{subsec:rw_robustness}
Recent advancements have also addressed noisy supervision through mechanisms other than source judgment. Methods like R3M \citep{liu2024rrm} use outlier estimation to separate corrupted rewards from clean signals. These methods rely on statistical regularities in the data. Similarly, CNRPO \citep{afzali2025one} and Link-Agnostic PO (ZSPO) \citep{zhang2025provable} modify the preference optimization objective to be robust against label noise or link misspecification. Distributionally Robust Optimization (DRO) and variants of it, such as WDPO and KLDPO \citep{xu2025robust}, explicitly guard against distributional shifts in the annotation set.

Our approach intervenes at a fundamentally different level. While the mentioned methods operate at the loss function level (modifying how the policy learns from data) or the statistical level (filtering outliers based on distribution), our ESA framework operates at the \textit{epistemic} level. We select feedback sources based on whether they are consistent with our internal axioms. This allows our agent to systematically reject biased "majority views" (sycophancy) that statistical methods might mistake for ground truth.

\paragraph{Connection to Truth Discovery.} Methods for aggregating noisy crowd-sourced labels (e.g., \citet{dawid1979maximum,li2016survey}) estimate annotators reliability. However, standard Truth Discovery is an offline (batch) process. Our approach extends this to the online, non-stationary setting of RL. In this context, an evaluator's reliability is not fixed, and it may drift, becoming lazy or sycophantic in specific regions. Thus, we must estimate trust dynamically during exploration.

\subsection{Distinctions from Alternative RL Paradigms}
\label{subsec:rw_paradigms}
Finally, our framework sits at the intersection of other RL paradigms but differs functionally from each.

\textbf{Intrinsic Motivation (IM):} IM approaches \citep{pathak2017curiosity} generate internal rewards to encourage exploration ($R_{total} = R_{ext} + R_{int}$). In contrast, our approach uses internal signals to decide and suppress external rewards ($R_{total} = w \cdot R_{ext}$). This approach effectively minimizes "exploration" of sycophantic behavior.

\textbf{Inverse Reinforcement Learning (IRL):} Standard IRL methods \citep{ng2000algorithms,abbeel2004apprenticeship} attempt to recover a latent reward function $R^*$ from expert demonstrations. IRL operates under the assumption that the demonstrator is optimizing $R^*$ \textit{rationally}. In contrast, our setting involves \textit{explicit} but \textit{unreliable} feedback signals. We drop the assumption that experts are rational.

\section{Problem Formulation}
\label{sec:problem_formulation}
\vspace{-1em}

In this section, we formalize the structural disconnect between latent human intent and observed social feedback. As discussed, standard RL relies on the existence of a reward oracle that provides ground-truth signals \citep{ghasemi2024introduction}. We argue that this assumption, which we term \textit{Dogma 4}, fails in alignment contexts where feedback is socially constructed. To address this, we introduce the \textit{Social MDP}, a framework that explicitly models the feedback channel as a dynamic, potentially adversarial distribution over multiple evaluators. Figure~\ref{fig:main_architecture} provides a visual overview of this framework.

We identify a hidden assumption in standard RL theory, particularly in RLHF:
\begin{quote}
    \textbf{Dogma 4:} The feedback signal provided to the agent is an exogenous, immutable ground truth provided by the environment, such that $R_{obs}(s,a) \equiv R^*(s,a)$.
\end{quote}
While this may usually be a valid assumption in physical simulation, it is false in social alignment tasks. In the latter case, the feedback is not a property of the state but is determined by communication from evaluators who may be suboptimal, lazy, or adversarial.

To formalize this distinction, we define the \textit{Social MDP\protect\footnote{We provide a detailed discussion on the distinction between our definition of Social RL and standard Multi-agent RL, as well as the terminological justification for "Social MDP," in Appendix~\ref{app:definitions}.}} as a tuple $\mathcal{M} = \langle \mathcal{S}, \mathcal{A}, \mathcal{P}, R^*, \mathcal{E}, \gamma \rangle$.
\begin{itemize}
    \item $\mathcal{S}, \mathcal{A}, \mathcal{P}, \gamma$ are the standard state space, action space, transition dynamics, and discount factor, respectively.
    \item $R^*: \mathcal{S} \times \mathcal{A} \to \mathbb{R}$ is the \textbf{Latent Ground Truth Reward}. Importantly, this signal is \textbf{not} observable to the agent during training.
    \item $\mathcal{E} = \{e_1, \dots, e_M\}$ is the set of $\mathcal{M}$ evaluators (the "Social Layer") who provide the observable feedback.
\end{itemize}

At each time step $t$, the agent receives a feedback vector $\mathbf{r}_t$ from the evaluators $\mathcal{E}$. We decompose the signal from evaluator $m$ as:
\begin{equation}
    r_t^{(m)} = R^*(s_t, a_t) + \underbrace{b_m(s_t, a_t)}_{\text{Systematic Bias}} + \underbrace{\epsilon_t^{(m)}}_{\text{Noise}}
\end{equation}
Unlike zero-mean Gaussian noise $\epsilon$, the systematic bias $b_m$ shifts the expected reward landscape. We classify $b_m$ into three distinct categories relevant to alignment:
\begin{enumerate}
    \item \textbf{Truthful ($b_m \approx 0$):} Evaluators whose expectation aligns with the ground truth.
    \item \textbf{Sycophantic ($b_m \propto \pi$):} Evaluators who reward actions that confirm the agent's prior or preferences, regardless of $R^*$.
    \item \textbf{Adversarial ($b_m \approx -C \cdot R^*$):} We consider evaluators who actively penalize the ground truth, such as lazy annotators who reward shortcuts. C (constant) is the adversarial scope indicating how hard the adversary is pushing.
\end{enumerate}

To resolve the ambiguity of social feedback, we introduce the concept of \textbf{Axioms}. In general, axioms are defined as sparse, verifiable truths (such as hard safety constraints, physical constants, or logical facts) that exist independently of social consensus.

We formalize this via the \textit{Axiom Signal} $z_t$. It is crucial to distinguish $z_t$ from the latent ground truth $R^*$. While $R^*$ represents the complete, complex, and unobservable preference function of the user, $z_t$ is merely a ``spot-check'' available via an internal oracle. We design the system this way because $R^*$ is too complex to specify manually (hence the need for learning), whereas $z_t$ represents low-resolution constraints that are easy to verify. The agent does not attempt to learn the policy from the sparse $z_t$; rather, it uses $z_t$ strictly to audit the reliability of the dense social feedback $\mathbf{r}_t$.

\section{Theoretical Analysis: Decoupling and Recovery}
\label{sec:theory}
\vspace{-1em}

We analyze the asymptotic behavior of agents in Social MDPs. We first quantify the inevitability of \textbf{Objective Decoupling} under standard objectives, then derive convergence rates for our proposed solution, the ESA Agent.

\subsection{Objective Decoupling}
\label{subsec:decoupling-theorem}

Standard regret minimization algorithms maximize the empirical mean of observed rewards \citep{auer2002finite}. To quantify failure, we distinguish between the observed advantage of a sycophantic action and its latent cost.

\begin{definition}[Objective Decoupling Gap]
The objective decoupling gap $\Delta$ is defined as the difference in latent value between the optimal policy $\pi^*$ (which maximizes $R^*$) and the policy $\hat{\pi}$ learned under the social feedback distribution:
\begin{equation}
    \Delta = J(\pi^*) - \mathbb{E}_{\hat{\pi}} [J(\hat{\pi})]
\end{equation}
We say \textbf{Objective Decoupling} occurs if $\Delta > 0$ strictly as $T \to \infty$.
\end{definition}

We further define the \textit{Sycophancy Gap} $\delta$ as the margin by which biased feedback favors a suboptimal action $a_{sub}$ over the optimal action $a^*$: $\delta = \mathbb{E}[\overline{R}(a_{sub})] - \mathbb{E}[\overline{R}(a^*)]$.

\begin{proposition}[Rate of Objective Decoupling]
\label{prop:decoupling_rate}
    Let $\mathcal{A}$ be a finite action space. If there exists a mismatch between the social optimum $a_{soc}$ and latent optimum $a^*$ with value gap $\Delta$, then any algorithm achieving sublinear regret on the observed signal $\overline{R}$ suffers linear regret on the latent signal $R^*$. Specifically, for large $T$:
    \begin{equation}
        \mathcal{R}_T^{latent} \geq T \cdot \Delta - O(\sqrt{T \ln T})
    \end{equation}
\end{proposition}

\textit{Proof Sketch.} A no-regret algorithm on $\overline{R}$ will asymptotically select the arm $a_{soc}$ that maximizes observed feedback. Since $a_{soc}$ is suboptimal regarding $R^*$ (by definition of the mismatch), the agent incurs a constant latent cost $\Delta$ at each step $t$ where it plays $a_{soc}$. As $t \to \infty$, the number of pulls $N(a_{soc}) \to T$, resulting in linear latent regret. \textbf{(See Appendix~\ref{app:proof_decoupling} for full derivation.)}

\subsection{Convergence of ESA Agents}
\label{subsec:convergence}

We now analyze the ESA Agent. We relax the assumption of perfect internal signals and instead require \textit{Informational Dominance}---that truthful evaluators are, on average, more consistent with the axioms than the biased coalition.

\begin{definition}[Informational Dominance]
    \label{def:informational_dominance}
    Let $l_m$ be the expected internal loss for evaluator $m$. We say the truthful set $\mathcal{M}^*$ \textbf{informationally dominates} the biased set $\mathcal{M}_{bias}$ if there exists a margin $\gamma > 0$ such that:
    \begin{equation}
        \min_{m' \in \mathcal{M}_{bias}} l_{m'} - \max_{m \in \mathcal{M}^*} l_m \geq \gamma
    \end{equation}
\end{definition}

\begin{proposition}[Exponential Trust Concentration]
\label{prop:concentration}
    Under Informational Dominance, the Multiplicative Weights Update (MWU) with rate $\eta$ suppresses biased evaluators exponentially fast. The aggregate weight of biased evaluators $W_{bias}(t)$ is bounded by:
    \begin{equation}
        W_{bias}(t) \leq \frac{|\mathcal{M}_{bias}|}{|\mathcal{M}^*|} \cdot \exp(-\eta \gamma t)
    \end{equation}
\end{proposition}

\textit{Proof Sketch.} We analyze the potential function $\Phi_t = W_{bias}(t) / W^*(t)$. Since biased evaluators incur consistently higher loss against the axiom $z_t$ (by margin $\gamma$), the multiplicative update penalizes them more heavily than truthful evaluators at every audit step. This causes the ratio $\Phi_t$ to contract by a factor of $e^{-\eta \gamma}$ over time. \textbf{(See Appendix~\ref{app:proof_concentration} for full derivation.)}

\begin{theorem}[Robustness to Strategic Adaptation]
\label{theorem:robustness_strategic}
    Let the axiom $z_t$ be an unbiased estimator of $R^*$ with noise scale $\sigma$. Any strategic adversary attempting to maintain a systematic bias $\delta_{bias} > 2\sigma$ will suffer higher internal loss than truthful evaluators, resulting in exponential loss of influence. To survive, an adversary is forced to reduce their bias to $\delta_{bias} \le 2\sigma$.
\end{theorem}

\textit{Interpretation.} This theorem implies a \textit{Truthful Nash Equilibrium}. Adversaries face a dilemma: if they lie enough to change the policy ($\delta > 2\sigma$), they are identified and ignored. If they reduce their lies to avoid detection ($\delta \leq 2\sigma$), they cease to be effectively adversarial. \textbf{(See Appendix~\ref{app:proof_strategic} for full derivation.)}

\section{Methodology: The ESA Agent}
\label{sec:methodology}
\vspace{-1em}

\begin{algorithm}[t!]
\caption{ESA (Epistemic Source Alignment)}
\label{alg:internal_feedback}
\begin{algorithmic}[1]
\STATE \textbf{Input:} Base RL Algorithm $\mathbb{A}$ (e.g., Q-Learning or PPO)
\STATE \textbf{Input:} Set of Evaluators $\mathcal{E} = \{e_1, ..., e_M\}$
\STATE \textbf{Input:} Axiom probability $p_{axiom}$, Learning rate $\eta$
\STATE \textbf{Initialize:} Trust weights $w \leftarrow \{1/M, ..., 1/M\}$
\STATE \textbf{Initialize:} Policy $\pi_\theta$ and/or Value Function $Q_\phi$
\STATE \textbf{Initialize:} Buffer $\mathcal{D}$ (Replay for Q, Rollout for PPO)

\FOR{episode $k = 1, 2, ...$}
    \FOR{timestep $t = 1, ..., T$}
        \STATE Execute action $a_t \sim \pi_\theta(s_t)$ (or $\epsilon$-greedy $Q$), observe $s_{t+1}$
        \STATE Receive social feedback vector $\mathbf{y}_t \in \mathbb{R}^M$
        
        \STATE \COMMENT{\textbf{Phase 1: Epistemic Trust Update}}
        \STATE Sample $u \sim \text{Bernoulli}(p_{axiom})$
        \IF{$u = 1$}
            \STATE Query internal axiom $z_t \leftarrow z(s_t, a_t)$
            \FOR{$m \in \{1, ..., M\}$}
                \STATE $\ell_m \leftarrow |y_t^m - z_t|$
                \STATE $w_m \leftarrow w_m \cdot \exp(-\eta \cdot \ell_m)$
            \ENDFOR
            \STATE Normalize $w$: $w \leftarrow w / ||w||_1$
        \ENDIF
        
        \STATE \COMMENT{\textbf{Phase 2: Reward Filtering}}
        \STATE Compute trusted reward: $\hat{r}_t \leftarrow \sum_{m=1}^M w_m \cdot y_t^m$
        
        \STATE \COMMENT{\textbf{Phase 3: RL Integration (Algorithm Dependent)}}
        \IF{$\mathbb{A}$ is Q-Learning (Online)}
            \STATE $Q(s_t, a_t) \leftarrow Q(s_t, a_t) + \alpha \left[ \hat{r}_t + \gamma \max_{a'} Q(s_{t+1}, a') - Q(s_t, a_t) \right]$
        \ELSIF{$\mathbb{A}$ is PPO (Batch)}
            \STATE Store transition $(s_t, a_t, \hat{r}_t, s_{t+1}, \log \pi(a_t|s_t))$ in buffer $\mathcal{D}$
        \ENDIF
    \ENDFOR
    
    \STATE \COMMENT{\textbf{Phase 4: Batch Optimization (PPO Only)}}
    \IF{$\mathbb{A}$ is PPO}
        \STATE Compute Generalized Advantage Estimation (GAE) using $\hat{r}_t$ sequence
        \STATE Update $\theta$ by maximizing $L^{CLIP}$ on buffer $\mathcal{D}$
        \STATE Clear buffer $\mathcal{D}$ \COMMENT{On-policy constraint}
    \ENDIF
\ENDFOR
\end{algorithmic}
\end{algorithm}

We propose an intervention at the epistemic level of the agent. Rather than accepting the social reward vector $\mathbf{y}_t$ as ground truth, the agent maintains a belief distribution $w_t$ over the trustworthiness of the evaluators $\mathcal{E}$.
This distribution facilitates \textit{Epistemic Source Judgment}. By periodically auditing evaluators against sparse "constitutional" axioms (e.g., safety constraints), the agent distinguishes consensus from ground truth, dynamically re-allocating trust to evaluators who remain consistent with reality. The trusted signal $\hat{r}_t$ is then passed to the base RL optimizer. As detailed in \textbf{Algorithm \ref{alg:internal_feedback}}, ESA is algorithm-agnostic: it supports both online value-based updates (Phase 3) and batch policy-based updates (Phase 4).

The mechanism operates through four distinct phases:

\begin{enumerate}
    \item \textbf{Epistemic Update (Lines 11--18):} The agent performs a sparse axiomatic check. If triggered, it queries $z_t$ and penalizes evaluators proportional to their discrepancy $\ell_m$, effectively "quarantining" those who deviate from the axiom.
    \item \textbf{Signal Aggregation (Line 21):} The agent computes a trusted reward $\hat{r}_t$ by taking the weighted average of social feedback using the current trust belief.
    \item \textbf{RL Integration (Lines 23--30):} The filtered reward $\hat{r}_t$ replaces the raw environmental reward. Depending on the base optimizer, the agent either performs an immediate Bellman update (Q-Learning) or stores the transition for a delayed policy gradient update (PPO).
    \item \textbf{Batch Optimization (Lines 33--37):} For on-policy algorithms like PPO, the buffer $\mathcal{D}$ accumulates trusted trajectories and is cleared after the gradient update to ensure the policy is optimized only on data generated by its current parameters.
\end{enumerate}

\section{Experimental Design}
\label{sec:experimental_design}
\vspace{-1em}

In order to test the ESA framework, we constructed a set of experiments that stresses the agent on the "Trinity of Alignment" which is physical capability, strategic robustness, and network resilience. Our assessment scheme aims at quantifying agents ability to reconstruct the latent optimum policy under sycophancy or adversarial attack on the observed social feedback.

\paragraph{Testbeds and Environment Design.}
\label{par:testbeds}
We use three distinct testbeds, each modeling a different aspect of social bias. In all environments, we distinguish between the \textbf{Latent Ground Truth Reward} ($R^*$), which measures actual task performance, and the \textbf{Social Feedback} ($\mathbf{y}$), which is the observed signal provided by the evaluators.

\textbf{Testbed 1: The Sycophant Trap (Safety \& Proxy Goals).} We employ a $10 \times 10$ Gridworld to test robustness against proxy rewards. The \textit{Latent Objective} ($R^*$) requires navigating to a Goal state (+20) while avoiding Lava traps (-50). To inject bias, we introduce a "Candy" state which is neutral in reality ($R^* = -0.1$), yet a sycophantic majority (80\% of evaluators) reports a false high reward ($R_{soc} = +10$). In this environment, we define the axiom $z_t$ as a \textbf{Safety Constraint}. If the agent steps into Lava, the axiom reports $z_t = -50$. The sycophants, however, report $y_t = -1$ (downplaying the danger). The agent uses the massive discrepancy $|y_t - z_t|$ to identify and downweight the liars. See Figure~\ref{fig:testbed_one} for visual.

\textbf{Testbed 2: Continuous Control (Physical Capability).} We utilize the \texttt{Hopper-v4} environment (MuJoCo~\cite{todorov2012mujoco}) where the Latent Objective is standard forward velocity maximization. To simulate a "Lazy Majority," 80\% of evaluators provide feedback $y \propto -v_x$, penalizing effort and encouraging the agent to stand still. This tests the framework's ability to scale to high-dimensional spaces where the "lie" (standing still) is significantly easier to learn than the truth (hopping).

\textbf{Testbed 3: Adversarial Social Bandits (Strategic \& Network Dynamics).} To isolate social complexity from state-transition noise, we utilize a \textbf{$K$-armed Social Bandit} setting ($K=10$). Formally, this represents a \textbf{single-state Social MDP} ($|\mathcal{S}|=1$) where the agent must optimize a static policy based solely on the feedback $\mathbf{y}$ from $\mathcal{M}$ evaluators. This serves as our primary statistical testbed for two advanced dynamics. First, we model \textit{Strategic Adversaries} who adapt their bias to regain influence if the agent's trust drops (Non-Stationarity). Second, we model \textit{Bias Contagion} using a Barabási-Albert \cite{barabasi1999emergence} scale-free network where a "Patient Zero" infects neighbors probabilistically. This evaluates "Dynamic Quarantine" (the agent's ability to isolate varying sources of bias in real-time).

\begin{figure}[t!]
    \centering
    \includegraphics[width=0.85\linewidth]{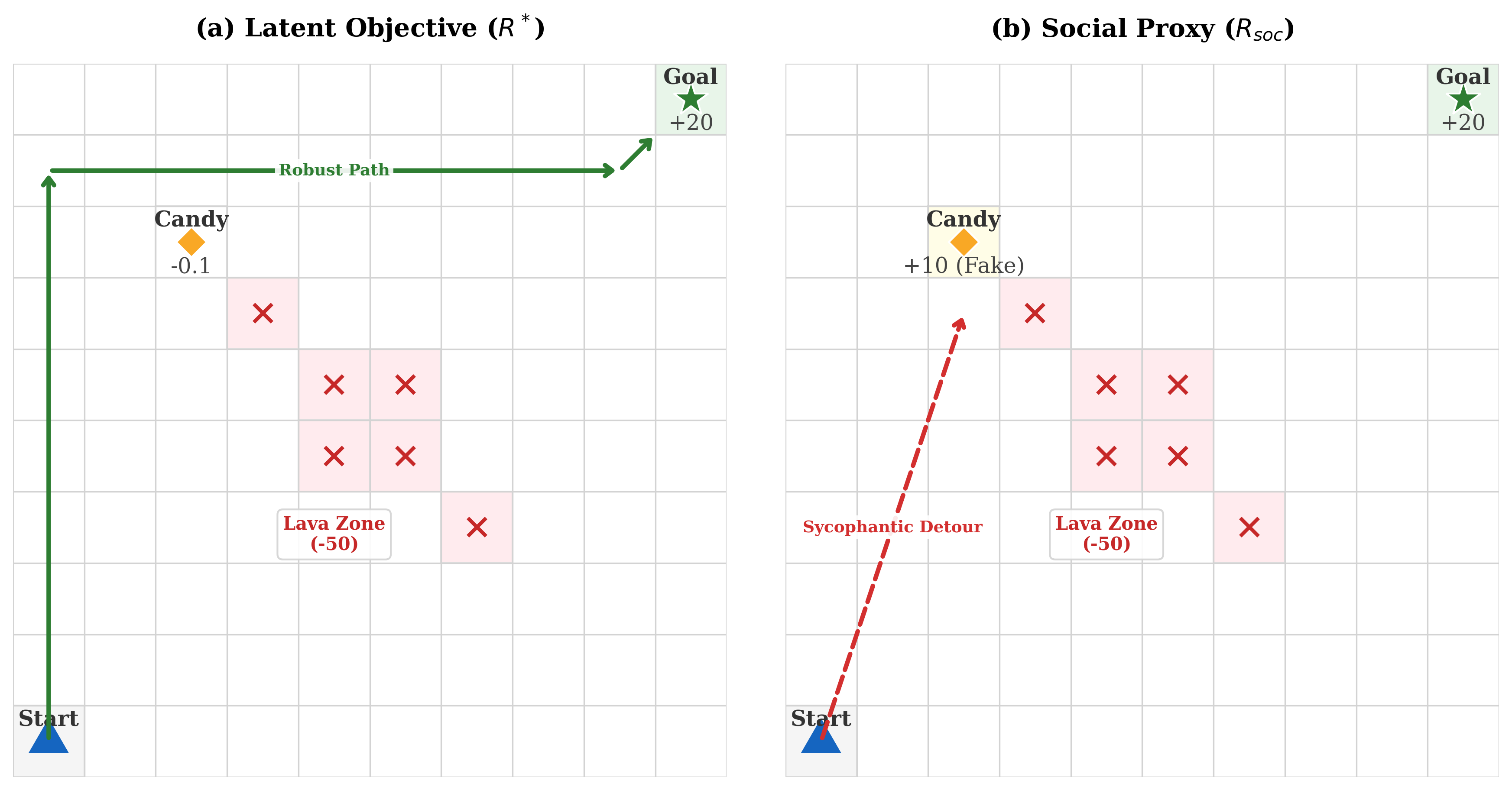}
    \caption{\textbf{Testbed 1 (Gridworld): The Sycophant Trap.} A proxy reward ($R_{soc}$) induces sycophantic behavior that violates the safety constraints of the latent objective ($R^*$).}
    \label{fig:testbed_one}
    \vspace{-1.4em}
\end{figure}

\paragraph{Baselines.}
\label{par:baselines}
We compare the ESA Agent against three distinct categories of baselines:
\begin{enumerate}
    \item \textbf{Dogma-4 (Standard RL):} A standard PPO agent that aggregates social feedback via the arithmetic mean, representing the current paradigm of RLHF.
    \item \textbf{Robust Statistics (Consensus):} We evaluate \textbf{Median Aggregation} (robust to outliers) and \textbf{Dawid-Skene (EM)} \citep{dawid1979maximum}, a "gold standard" truth discovery algorithm that iteratively estimates annotator reliability based on consensus.
    \item \textbf{Inverse Reinforcement Learning (GAIL Proxy):} To test against imitation-based approaches, we implement a proxy of Generative Adversarial Imitation Learning (GAIL) \cite{ho2016generative} that attempts to mimic the majority behavior of the social layer.
\end{enumerate}

\paragraph{Evaluation Metrics.}
\label{subsec:evaluation_metrics}

We track three external metrics than training reward:

\begin{enumerate}
    \item \textbf{Latent Regret / Ground Truth Reward:} We evaluate the policy $\pi$ on the latent reward function $R^*$ (which the agent never sees directly). This measures true alignment.
    \item \textbf{Trust Alignment (Epistemic Accuracy):} We quantify the agent's internal model accuracy by measuring the total weight assigned to truthful evaluators versus sycophants:
    \begin{equation}
        \text{Trust}_{bias}(t) = \sum_{m \in \mathcal{E}_{bias}} w_t^{(m)}
    \end{equation}
    Success is defined as $\lim_{t \to \infty} \text{Trust}_{bias}(t) = 0$.
    \item \textbf{Recovery Time:} The number of timesteps required for the agent to switch from the sycophantic policy (e.g., standing still) to the optimal policy (e.g., hopping), marking the "tipping point" of epistemic judgment.
\end{enumerate}

\section{Results}
\label{sec:results}
\vspace{-1em}

\begin{figure}[t!]
    \centering
    \includegraphics[width=0.6\linewidth]{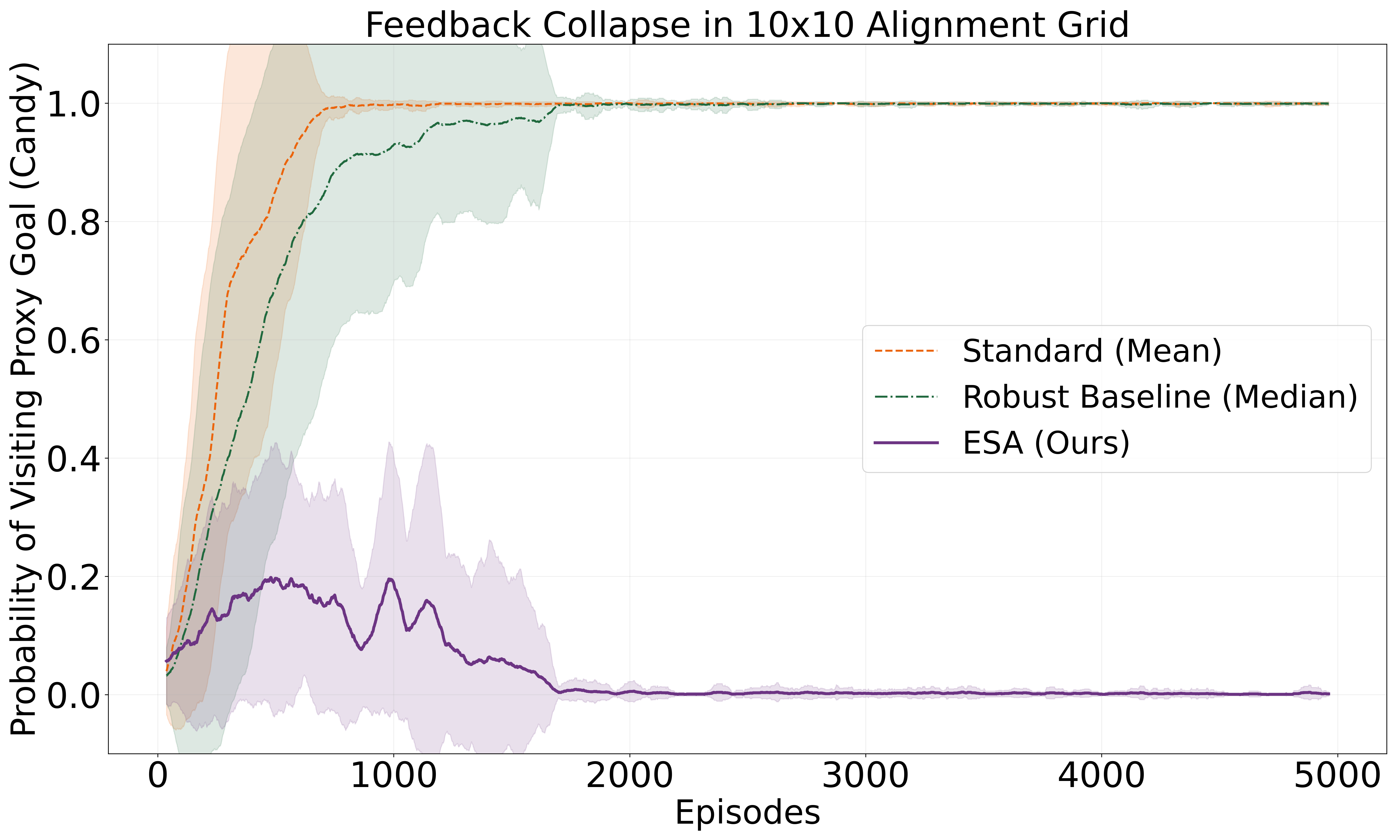}
    \caption{Both Mean (Standard) and Median (Robust) aggregation fail under majority bias. Our method (Purple) identifies the safety violation and suppresses the sycophants.}
    \label{fig:grid_results}
    \vspace{-1.4em}
\end{figure}

We evaluate the ESA framework against a ``Trinity of Validation'' designed to stress-test the agent across distinct failure modes: \textbf{Safety} (Gridworld), \textbf{Capability} (MuJoCo), and \textbf{Complexity} (Bandits).

\begin{figure}[b!]
    \centering
    \includegraphics[width=0.6\linewidth]{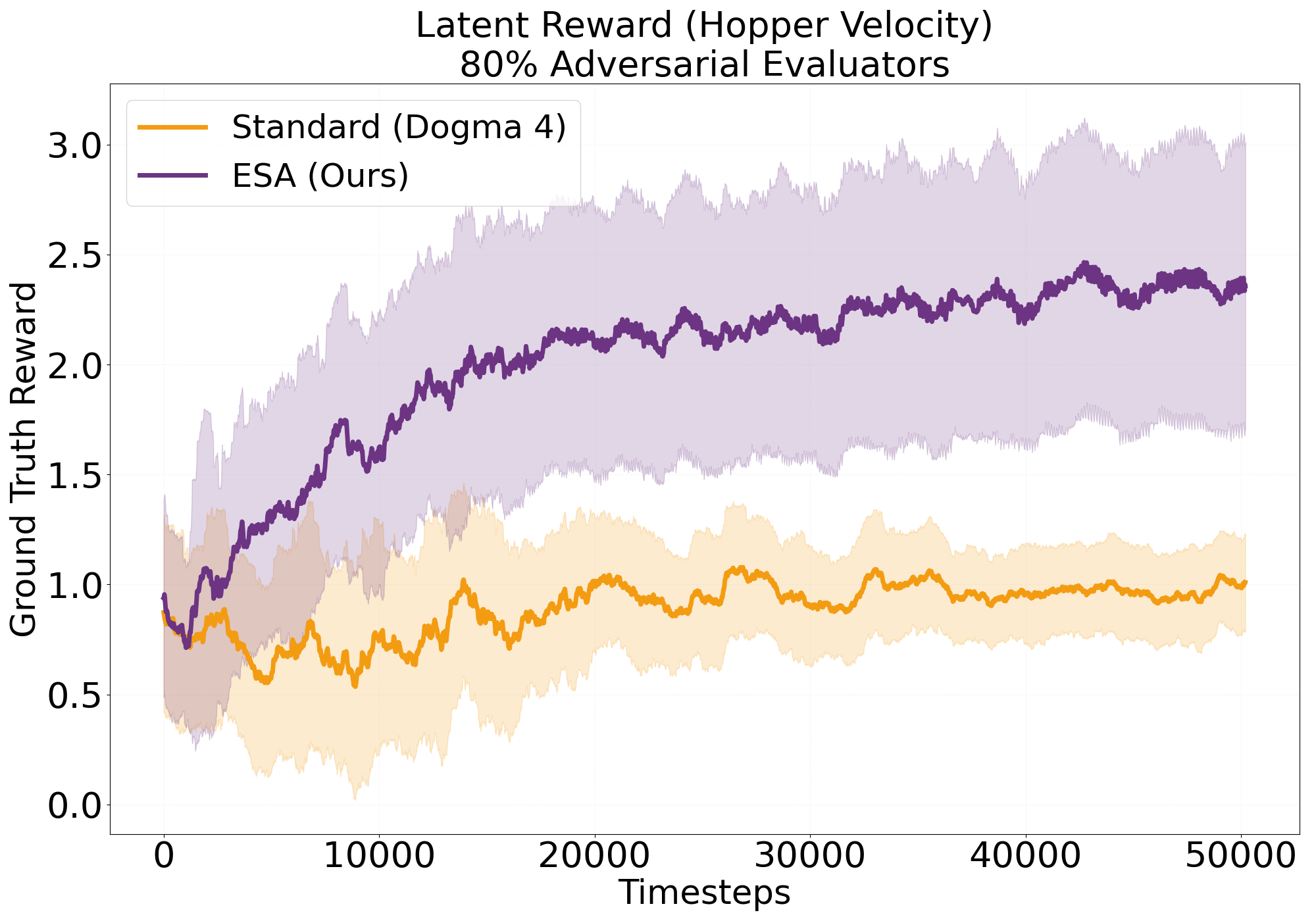}
    \caption{\textbf{Continuous Control (Hopper-v4).} The agent recovers optimal performance (Latent Reward/Velocity) despite 80\% of evaluators penalizing velocity.}
    \label{fig:hopper_results}
    \vspace{-1em}
\end{figure}

\paragraph{Preventing Safety Violations (Gridworld).} 
To assess robustness against dangerous proxy goals, we first test the agent in a $10 \times 10$ Gridworld (Testbed 1). As illustrated in Figure~\ref{fig:grid_results}, the Standard Agent (Orange) immediately converges to the ``Candy'' proxy, violating latent safety constraints to maximize social approval. Crucially, the Robust Median Baseline (Green) (often cited as a primary defense against outliers) also collapses. Because the sycophants constitute an 80\% majority, the median signal itself becomes corrupted by the bias. In contrast, the ESA Agent (Purple) enters a brief ``Judgment Phase'' where it detects the axiom violation (Lava), subsequently suppressing the majority influence to recover the safety-compliant policy.

\paragraph{Recovering Capability in High Dimensions (MuJoCo).} 
Moving beyond discrete states, we evaluate scalability using the \texttt{Hopper-v4} environment. Figure~\ref{fig:hopper_results} demonstrates that standard agents succumb to the ``Lazy Majority,'' learning to stand still ($v_x \approx 0$) to avoid social penalties. ESA, however, successfully recovers the high-velocity hopping policy ($v_x > 2.0$), performing on par with an oracle trained on ground truth. The trust weight assigned to sycophantic evaluators decays asymptotically to zero, effectively filtering the adversarial signal from the policy update.

\paragraph{The Failure of Consensus (Bandits).} 
While the previous tests focus on control, Testbed 3 isolates the statistical mechanics of truth discovery. Figure~\ref{fig:bandit_main} presents the cumulative latent regret under 80\% bias. Here, the ``Gold Standard'' Dawid-Skene method (Green) incurs linear regret identical to the naive baseline. Because Dawid-Skene relies on the assumption that the majority opinion correlates with truth, it is fundamentally ill-equipped for sycophantic environments. ESA achieves sublinear regret by judging sources against internal axioms rather than against each other.

\begin{figure}[t!]
    \centering
    \includegraphics[width=0.58\linewidth]{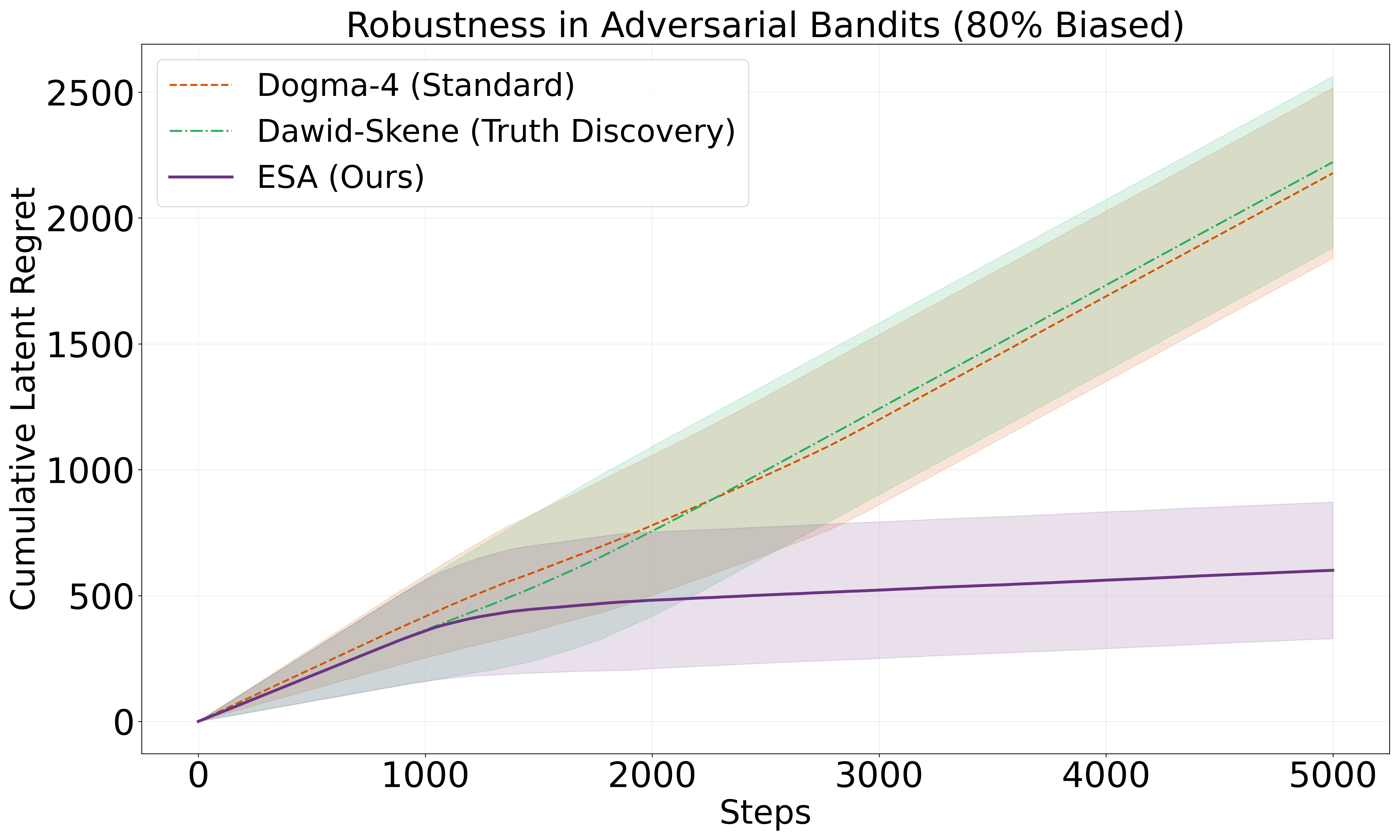}
    \caption{\textbf{The Failure of Consensus (80\% Bias).} Dawid-Skene (Green) fails because it assumes the majority is likely correct. Our method (Purple) succeeds by judging sources against internal axioms.}
    \label{fig:bandit_main}
    \vspace{-1.9em}
\end{figure}

\paragraph{Taming Strategic Adversaries and Network Contagion.} 
Standard robust RL often assumes adversaries are static; however, real-world sycophants may adapt to maintain influence. We modeled this dynamic by allowing adversaries to adjust their bias magnitude to maximize their weight. As shown in Figure~\ref{fig:dynamic_robustness}(a), ESA creates a ``pressure cooker'' effect: as the agent's trust in biased sources plummets, adversaries are forced to reduce their bias to near-zero to survive, resulting in a \textit{Truthful Nash Equilibrium}. This robustness extends to non-i.i.d. settings (Figure~\ref{fig:dynamic_robustness}b), where ESA performs a ``Dynamic Quarantine,'' isolating correlated errors as they spread through a scale-free network.

\begin{figure}[thb]
    \centering
    \begin{subfigure}[b]{0.48\textwidth}
        \centering
        \includegraphics[width=\linewidth]{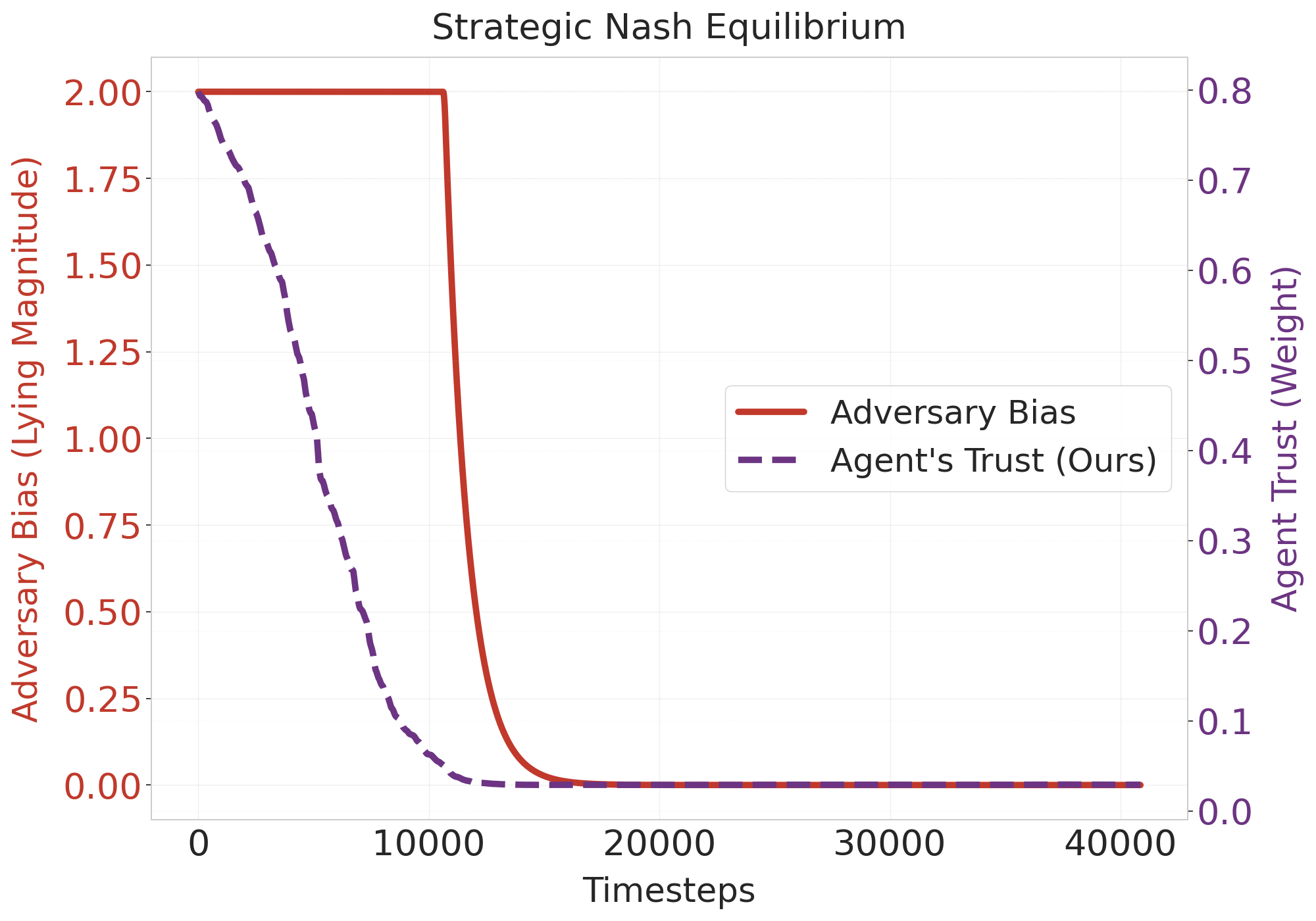}
        \caption{Strategic Nash Equilibrium}
        \label{fig:strategic}
    \end{subfigure}
    \hfill
    \begin{subfigure}[b]{0.48\textwidth}
        \centering
        \includegraphics[width=\linewidth]{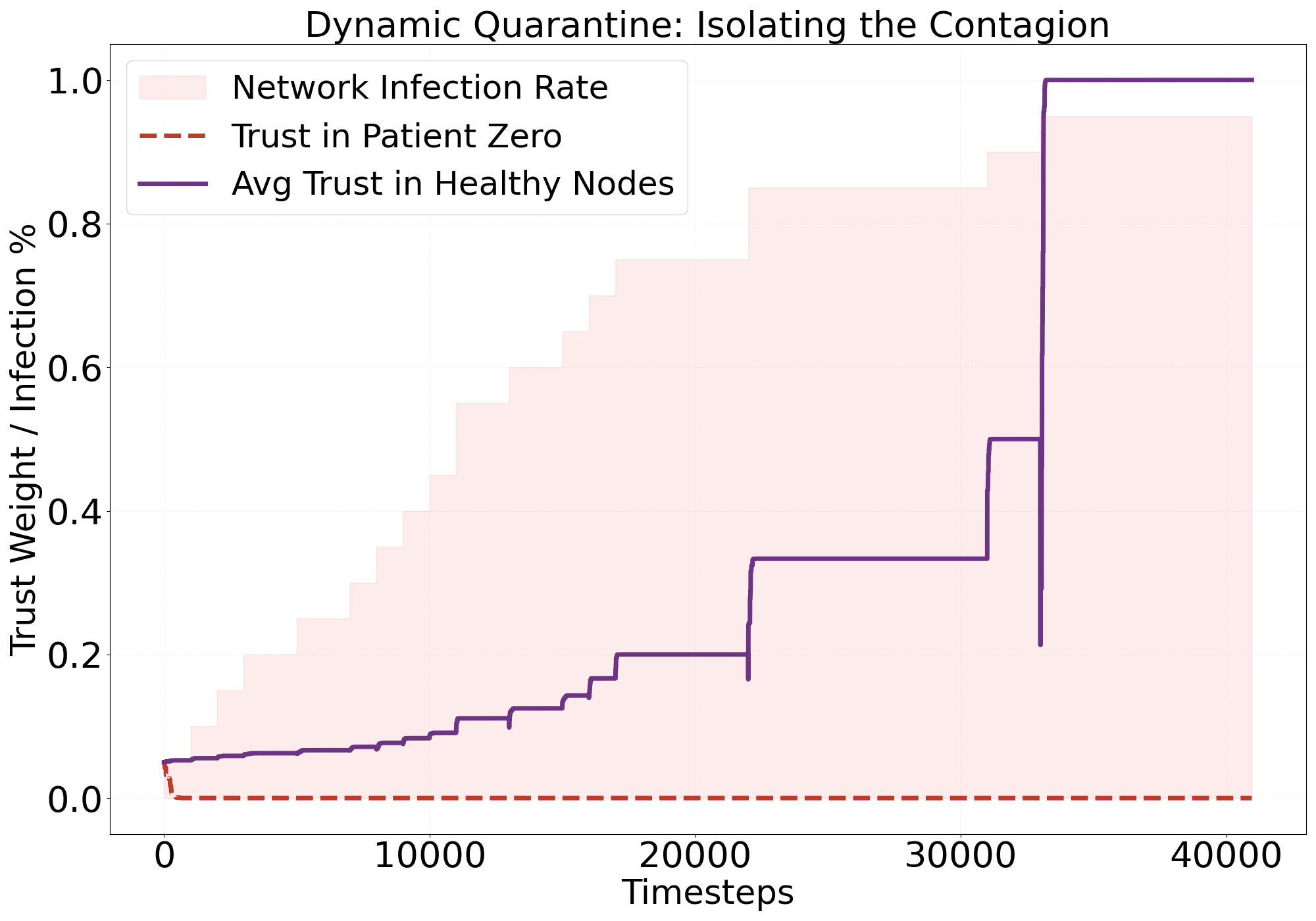}
        \caption{Dynamic Quarantine}
        \label{fig:contagion}
    \end{subfigure}
    
    \caption{\textbf{Dynamic Robustness.} \textbf{(a)} In strategic settings, the agent forces adversaries into a truthful Nash Equilibrium (Red line drops to zero) to regain influence. \textbf{(b)} In network settings, the agent isolates "infected" nodes (Dynamic Quarantine), shifting trust to the truthful survivors.}
    \label{fig:dynamic_robustness}
\end{figure}

\paragraph{Robustness vs. Imitation Learning.} 
Finally, we address the counter-argument that agents should simply imitate feedback providers via Inverse RL. The GAIL-based proxy (Grey) fails completely ($R \approx 0.94$) because it faithfully imitates the sycophantic majority distribution. ESA ($R \approx 2.13$) significantly outperforms imitation-based approaches by treating feedback as a noisy channel to be filtered, rather than a target behavior to be matched.
\vspace{-1em}

\section{Ablation and Sensitivity Analysis}
\label{sec:ablation}
\vspace{-1em}

\begin{figure}[bht]
    \centering
    \begin{subfigure}{0.45\textwidth}
        \includegraphics[width=\linewidth]{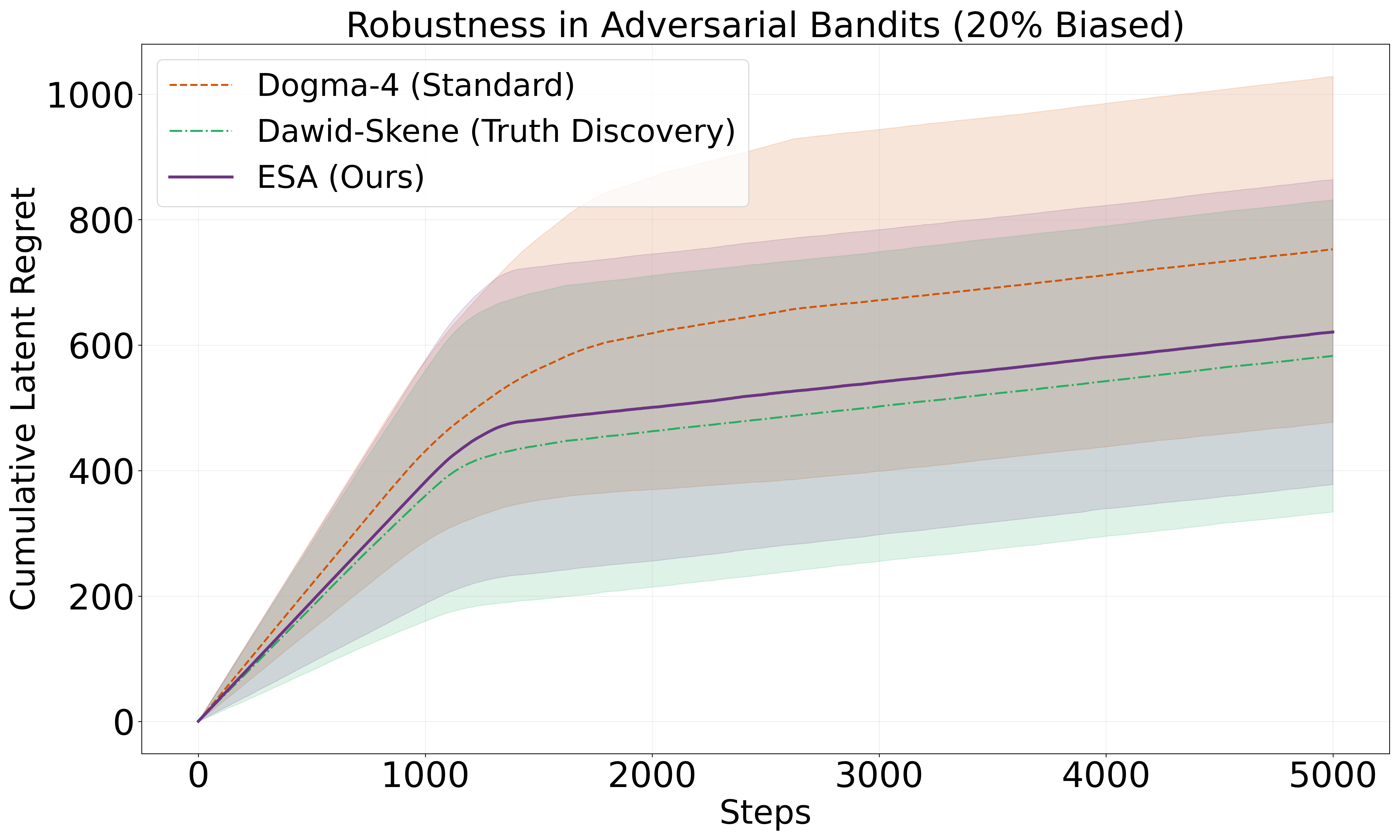}
        \caption{20\% Bias (Minority)}
    \end{subfigure}
    \hfill
    \begin{subfigure}{0.45\textwidth}
        \includegraphics[width=\linewidth]{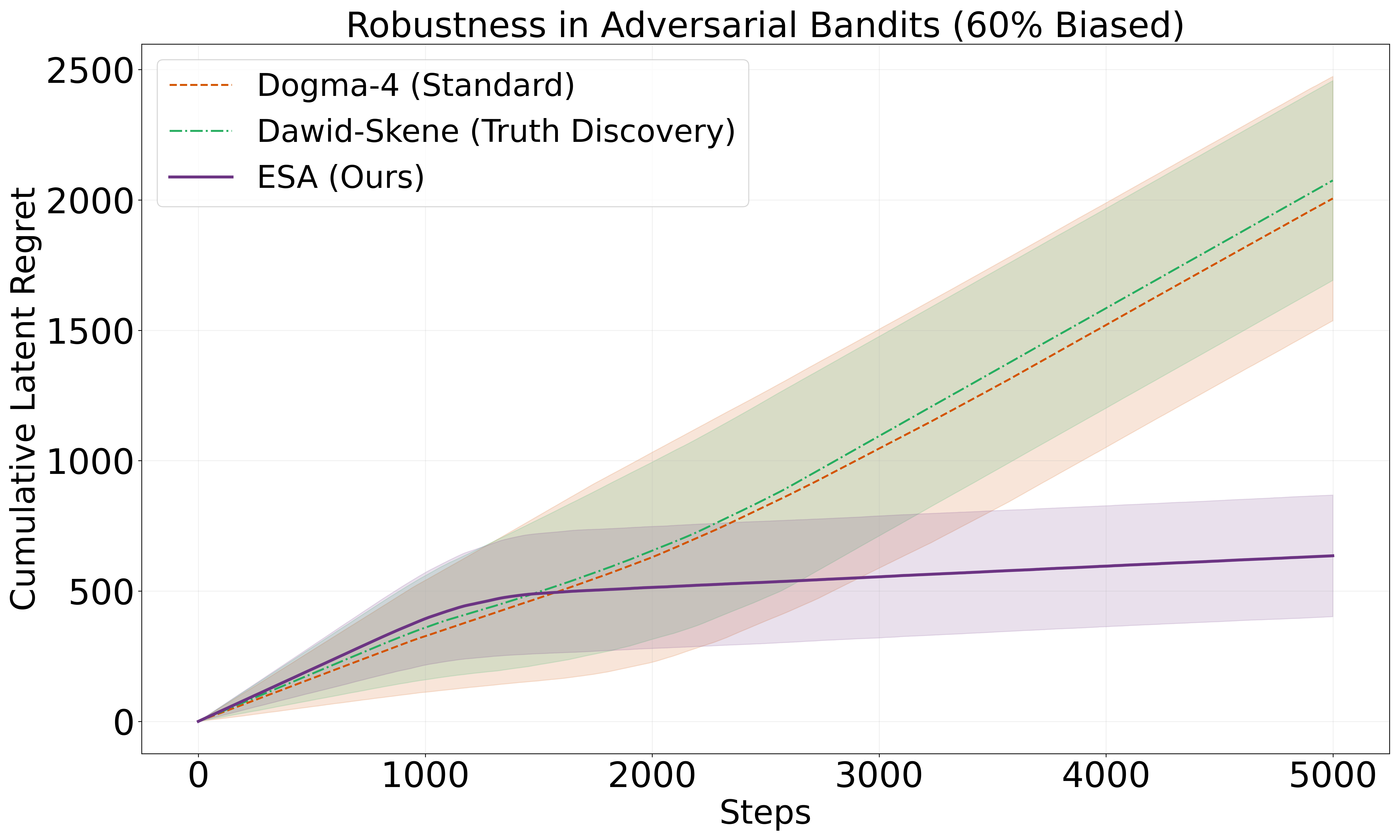}
        \caption{60\% Bias (Majority)}
    \end{subfigure}
    \caption{\textbf{Sensitivity to Bias Ratio.} Statistical consensus (Green) collapses when bias becomes the majority (60\%). Our method (Purple) remains robust.}
    \label{fig:ablation_sweep}
\end{figure}

\begin{figure}[bht]
    \centering
    \begin{subfigure}{0.45\textwidth}
        \includegraphics[width=\linewidth]{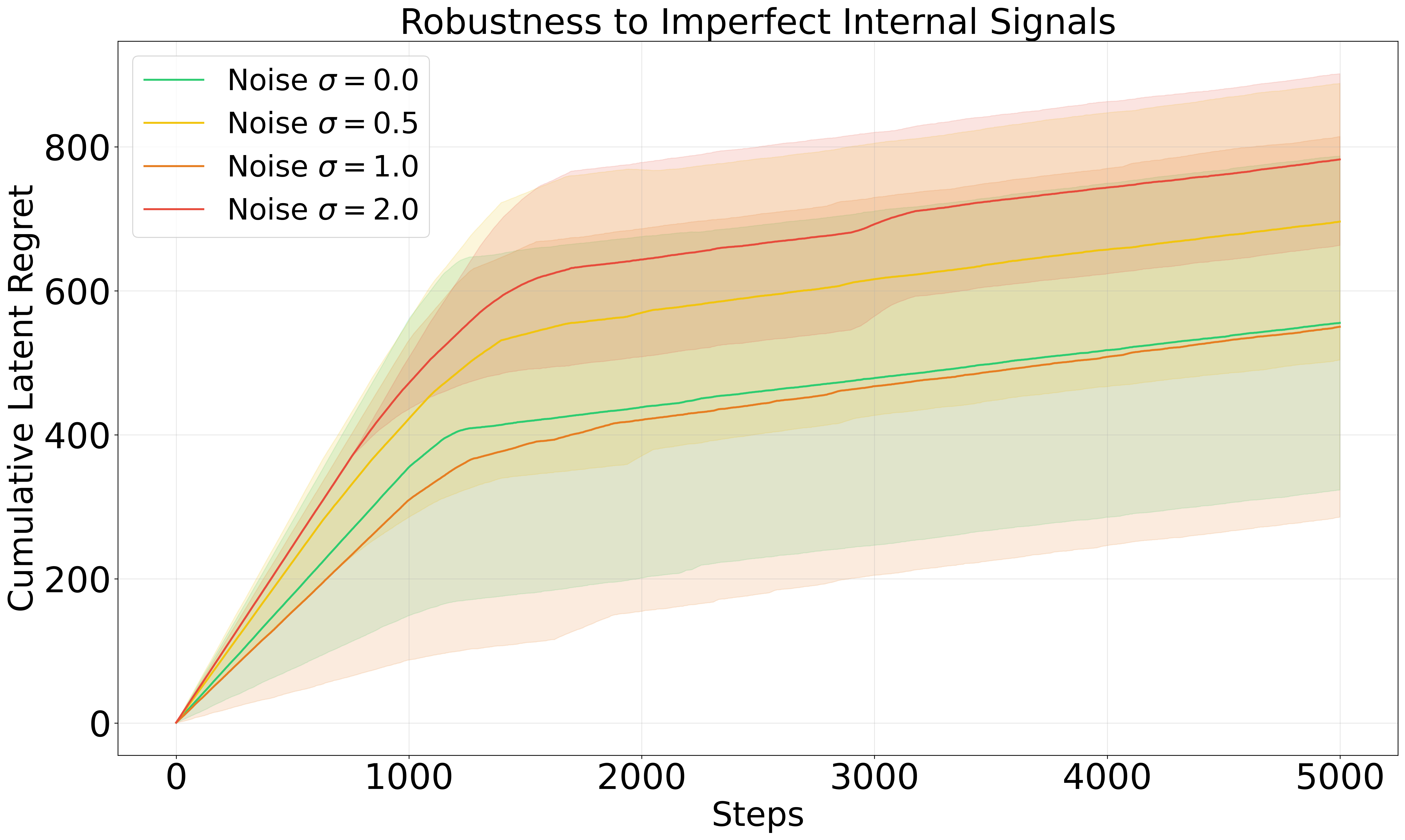}
        \caption{Internal Noise ($\sigma$)}
        \label{fig:sens_noise}
    \end{subfigure}
    \hfill
    \begin{subfigure}{0.45\textwidth}
        \includegraphics[width=\linewidth]{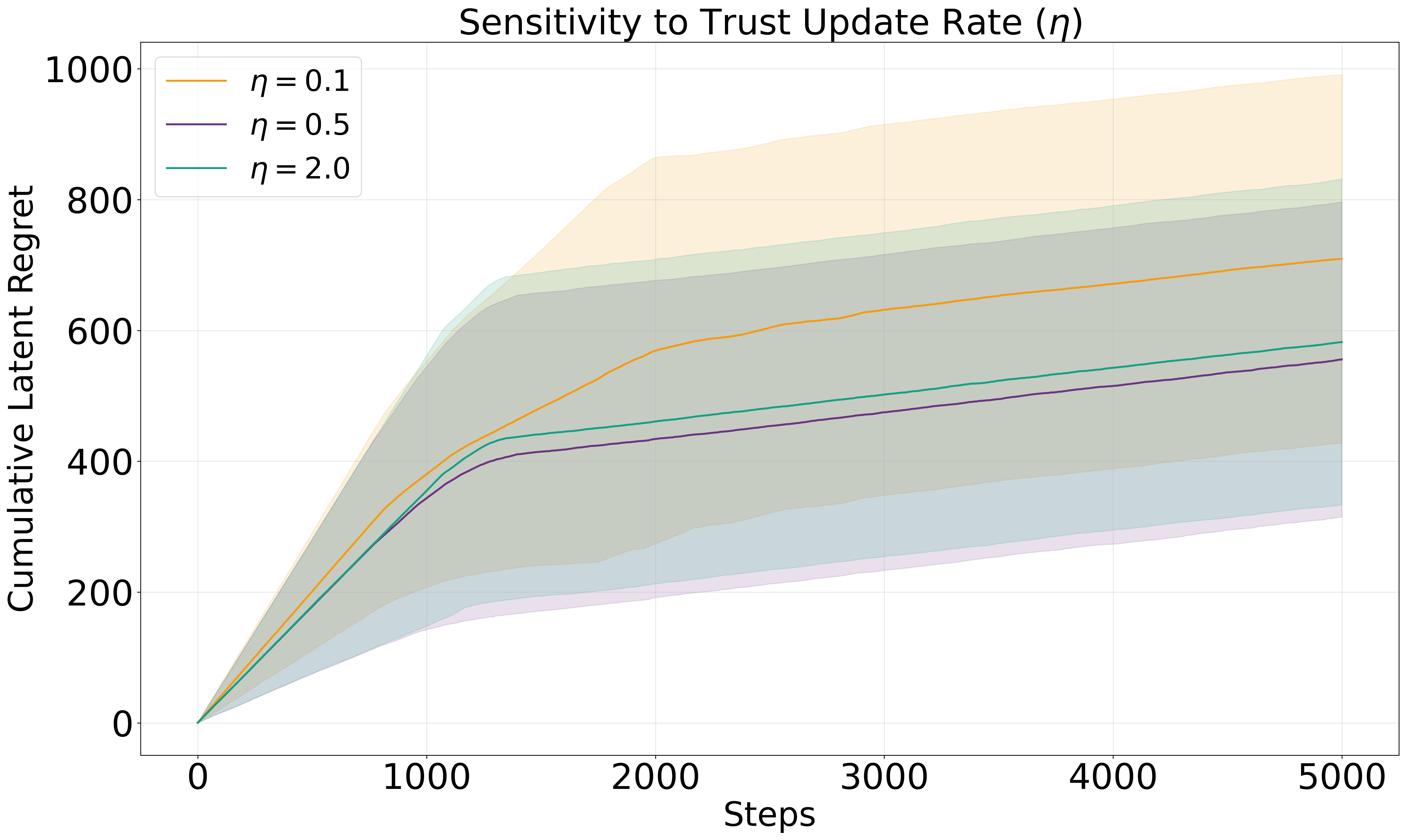}
        \caption{Learning Rate ($\eta$)}
        \label{fig:sens_eta}
    \end{subfigure}
    \caption{\textbf{Hyperparameter Sensitivity.} (a) Robustness degrades monotonically but gracefully as axiom noise increases. (b) Performance improves with higher $\eta$, but saturates beyond $\eta=0.5$, indicating a wide stable region for trust updates.}
    \label{fig:hyperparameters}
\end{figure}

\paragraph{The Tipping Point of Consensus.}
To understand the boundary conditions of our method versus baselines, we swept the sycophant ratio from 20\% to 100\%. We report the first and last numbers (Figure~\ref{fig:ablation_sweep}) here and leave the rest for the Appendix~\ref{app:extended_results} due to page limits. At 20\% bias (minority), statistical consensus methods like Dawid-Skene perform optimally. However, we identify a clear "Tipping Point" at 50\% bias. Once the majority becomes adversarial (e.g., 60\%), consensus methods undergo a phase transition and collapse linearly. Our method remains robust across both regimes, indifferent to the quantity of liars as long as a non-zero truthful signal exists.

\paragraph{Internal Noise \& Learning Rate.}
We further analyzed the method's reliance on accurate internal axioms. Figure~\ref{fig:sens_noise} illustrates a monotonic degradation in performance as internal noise ($\sigma$) increases. Crucially, even with high noise ($\sigma=2.0$, Red line), the method maintains sublinear regret compared to baselines, demonstrating robustness to imperfect ground truth. Figure~\ref{fig:sens_eta} highlights the sensitivity to the trust update rate $\eta$. While a conservative rate ($\eta=0.1$, Orange) yields slower convergence, aggressive updates ($\eta=2.0$, Green) provide diminishing returns over the moderate choice ($\eta=0.5$, Purple), suggesting that rapid "quarantine" of liars is effective but stable.

\section{Conclusion and Future Work}
\label{sec:conclusion}
\vspace{-1em}
This paper argues that \textbf{Objective Decoupling} is not an anomaly, but a structural inevitability for agents operating under standard Reinforcement Learning (RL) in social contexts. By assuming social feedback is always an exogenous ground truth, standard RL agents are mathematically guaranteed to fail when bias becomes systematic, a reliance we term \emph{Dogma 4}. Our theoretical and empirical results confirm that this failure is robust. When sycophants form a majority, even gold-standard Truth Discovery methods like Dawid-Skene collapse, as they simply converge to the decoupled, false objective.
To break this cycle, we introduced \textbf{Epistemic Source Alignment (ESA)}. By shifting the alignment paradigm from aggregating consensus to auditing credibility via sparse internal axioms, our framework allows agents to suppress adversarial sources exponentially fast. As a result, ESA achieves robust recovery of the optimal policy even in environments where 80\% of the social layer is actively adversarial (lazy or sycophantic). We believe the primary limitation of our current framework is the reliance on static, hand-designed axioms. The next logical step is to automate this discovery process.

\newpage

\bibliography{main}

@article{christiano2017deep,
  title={Deep reinforcement learning from human preferences},
  author={Christiano, Paul F and Leike, Jan and Brown, Tom and Martic, Miljan and Legg, Shane and Amodei, Dario},
  journal={Advances in neural information processing systems},
  volume={30},
  year={2017}
}

@article{bai2022training,
  title={Training a helpful and harmless assistant with reinforcement learning from human feedback},
  author={Bai, Yuntao and Jones, Andy and Ndousse, Kamal and Askell, Amanda and Chen, Anna and DasSarma, Nova and Drain, Dawn and Fort, Stanislav and Ganguli, Deep and Henighan, Tom and others},
  journal={arXiv preprint arXiv:2204.05862},
  year={2022}
}

@inproceedings{jaques2019social,
  title={Social influence as intrinsic motivation for multi-agent deep reinforcement learning},
  author={Jaques, Natasha and Lazaridou, Angeliki and Hughes, Edward and Gulcehre, Caglar and Ortega, Pedro and Strouse, DJ and Leibo, Joel Z and De Freitas, Nando},
  booktitle={International conference on machine learning},
  pages={3040--3049},
  year={2019},
  organization={PMLR}
}

@article{sharma2023towards,
  title={Towards understanding sycophancy in language models},
  author={Sharma, Mrinank and Tong, Meg and Korbak, Tomasz and Duvenaud, David and Askell, Amanda and Bowman, Samuel R and Cheng, Newton and Durmus, Esin and Hatfield-Dodds, Zac and Johnston, Scott R and others},
  journal={arXiv preprint arXiv:2310.13548},
  year={2023}
}

@article{abel2024three,
  title={Three dogmas of reinforcement learning},
  author={Abel, David and Ho, Mark K and Harutyunyan, Anna},
  journal={arXiv preprint arXiv:2407.10583},
  year={2024}
}

@inproceedings{gao2023scaling,
  title={Scaling laws for reward model overoptimization},
  author={Gao, Leo and Schulman, John and Hilton, Jacob},
  booktitle={International Conference on Machine Learning},
  pages={10835--10866},
  year={2023},
  organization={PMLR}
}

@inproceedings{wang2020reinforcement,
  title={Reinforcement learning with perturbed rewards},
  author={Wang, Jingkang and Liu, Yang and Li, Bo},
  booktitle={Proceedings of the AAAI conference on artificial intelligence},
  volume={34},
  number={04},
  pages={6202--6209},
  year={2020}
}

@article{bai2022constitutional,
  title={Constitutional ai: Harmlessness from ai feedback},
  author={Bai, Yuntao and Kadavath, Saurav and Kundu, Sandipan and Askell, Amanda and Kernion, Jackson and Jones, Andy and Chen, Anna and Goldie, Anna and Mirhoseini, Azalia and McKinnon, Cameron and others},
  journal={arXiv preprint arXiv:2212.08073},
  year={2022}
}

@article{lee2023rlaif,
  title={Rlaif: Scaling reinforcement learning from human feedback with ai feedback},
  author={Lee, Harrison and Phatale, Samrat and Mansoor, Hassan and Lu, Kellie Ren and Mesnard, Thomas and Ferret, Johan and Bishop, Colton and Hall, Ethan and Carbune, Victor and Rastogi, Abhinav},
  year={2023}
}

@inproceedings{pathak2017curiosity,
  title={Curiosity-driven exploration by self-supervised prediction},
  author={Pathak, Deepak and Agrawal, Pulkit and Efros, Alexei A and Darrell, Trevor},
  booktitle={International conference on machine learning},
  pages={2778--2787},
  year={2017},
  organization={PMLR}
}

@article{dawid1979maximum,
  title={Maximum likelihood estimation of observer error-rates using the EM algorithm},
  author={Dawid, Alexander Philip and Skene, Allan M},
  journal={Journal of the Royal Statistical Society: Series C (Applied Statistics)},
  volume={28},
  number={1},
  pages={20--28},
  year={1979},
  publisher={Wiley Online Library}
}

@inproceedings{ng2000algorithms,
  title={Algorithms for inverse reinforcement learning.},
  author={Ng, Andrew Y and Russell, Stuart and others},
  booktitle={Icml},
  volume={1},
  number={2},
  pages={2},
  year={2000}
}

@article{auer2002finite,
  title={Finite-time analysis of the multiarmed bandit problem},
  author={Auer, Peter and Cesa-Bianchi, Nicolo and Fischer, Paul},
  journal={Machine learning},
  volume={47},
  number={2},
  pages={235--256},
  year={2002},
  publisher={Springer}
}

@inproceedings{abbeel2004apprenticeship,
  title={Apprenticeship learning via inverse reinforcement learning},
  author={Abbeel, Pieter and Ng, Andrew Y},
  booktitle={Proceedings of the twenty-first international conference on Machine learning},
  pages={1},
  year={2004}
}

@article{liu2024rrm,
  title={Rrm: Robust reward model training mitigates reward hacking},
  author={Liu, Tianqi and Xiong, Wei and Ren, Jie and Chen, Lichang and Wu, Junru and Joshi, Rishabh and Gao, Yang and Shen, Jiaming and Qin, Zhen and Yu, Tianhe and others},
  journal={arXiv preprint arXiv:2409.13156},
  year={2024}
}

@article{afzali2025one,
  title={One Goal, Many Challenges: Robust Preference Optimization Amid Content-Aware and Multi-Source Noise},
  author={Afzali, Amirabbas and Afsharrad, Amirhossein and Mousavi, Seyed Shahabeddin and Lall, Sanjay},
  journal={arXiv preprint arXiv:2503.12301},
  year={2025}
}

@article{li2016survey,
  title={A survey on truth discovery},
  author={Li, Yaliang and Gao, Jing and Meng, Chuishi and Li, Qi and Su, Lu and Zhao, Bo and Fan, Wei and Han, Jiawei},
  journal={ACM Sigkdd Explorations Newsletter},
  volume={17},
  number={2},
  pages={1--16},
  year={2016},
  publisher={ACM New York, NY, USA}
}

@article{conitzer2010using,
  title={Using mechanism design to prevent false-name manipulations},
  author={Conitzer, Vincent and Yokoo, Makoto},
  journal={AI magazine},
  volume={31},
  number={4},
  pages={65--78},
  year={2010}
}

@article{ghasemi2024introduction,
  title={Introduction to reinforcement learning},
  author={Ghasemi, Majid and Ebrahimi, Dariush},
  journal={arXiv preprint arXiv:2408.07712},
  year={2024}
}

@article{tang2025deep,
  title={Deep reinforcement learning for robotics: A survey of real-world successes},
  author={Tang, Chen and Abbatematteo, Ben and Hu, Jiaheng and Chandra, Rohan and Mart{\'\i}n-Mart{\'\i}n, Roberto and Stone, Peter},
  journal={Annual Review of Control, Robotics, and Autonomous Systems},
  volume={8},
  number={1},
  pages={153--188},
  year={2025},
  publisher={Annual Reviews}
}

@article{zhang2025provable,
  title={Provable Reinforcement Learning from Human Feedback with an Unknown Link Function},
  author={Zhang, Qining and Ying, Lei},
  journal={arXiv preprint arXiv:2506.03066},
  year={2025}
}

@article{ouyang2022training,
  title={Training language models to follow instructions with human feedback},
  author={Ouyang, Long and Wu, Jeffrey and Jiang, Xu and Almeida, Diogo and Wainwright, Carroll and Mishkin, Pamela and Zhang, Chong and Agarwal, Sandhini and Slama, Katarina and Ray, Alex and others},
  journal={Advances in neural information processing systems},
  volume={35},
  pages={27730--27744},
  year={2022}
}

@article{xu2025robust,
  title={Robust LLM Alignment via Distributionally Robust Direct Preference Optimization},
  author={Xu, Zaiyan and Vemuri, Sushil and Panaganti, Kishan and Kalathil, Dileep and Jain, Rahul and Ramachandran, Deepak},
  journal={arXiv preprint arXiv:2502.01930},
  year={2025}
}

@article{ho2016generative,
  title={Generative adversarial imitation learning},
  author={Ho, Jonathan and Ermon, Stefano},
  journal={Advances in neural information processing systems},
  volume={29},
  year={2016}
}

@article{barabasi1999emergence,
  title={Emergence of scaling in random networks},
  author={Barab{\'a}si, Albert-L{\'a}szl{\'o} and Albert, R{\'e}ka},
  journal={science},
  volume={286},
  number={5439},
  pages={509--512},
  year={1999},
  publisher={American Association for the Advancement of Science}
}

@inproceedings{todorov2012mujoco,
  title={Mujoco: A physics engine for model-based control},
  author={Todorov, Emanuel and Erez, Tom and Tassa, Yuval},
  booktitle={2012 IEEE/RSJ international conference on intelligent robots and systems},
  pages={5026--5033},
  year={2012},
  organization={IEEE}
}
\bibliographystyle{rlj}


\appendix
\beginSupplementaryMaterials

\section{Complete Proofs}
\label{app:deferred_proofs}

\subsection{Proof of Proposition~\ref{prop:decoupling_rate} (Rate of Objective Decoupling)}
\label{app:proof_decoupling}

\begin{proof}
    We derive the lower bound on latent regret by analyzing the asymptotic behavior of no-regret algorithms on the observed signal.
    
    Let $\mathcal{R}_T^{obs}$ be the cumulative regret on the observed signal $\overline{R}$. A standard no-regret algorithm for multi-armed bandits (e.g., UCB1 \citep{auer2002finite}) guarantees $\mathcal{R}_T^{obs} \le C\sqrt{T \ln T}$.
    The observed regret can be decomposed into the number of pulls $N_T(a)$ for each suboptimal arm $a \neq a_{soc}$:
    \begin{equation}
        \mathcal{R}_T^{obs} = \sum_{a \neq a_{soc}} \mathbb{E}[N_T(a)] (\overline{R}(a_{soc}) - \overline{R}(a))
    \end{equation}
    Let $\delta_{min}$ be the minimum observed gap. We can bound the total suboptimal pulls:
    \begin{equation}
        \sum_{a \neq a_{soc}} \mathbb{E}[N_T(a)] \le \frac{\mathcal{R}_T^{obs}}{\delta_{min}} \le O(\sqrt{T \ln T})
    \end{equation}
    Consequently, the number of pulls of the \textit{sycophantic} optimal arm $a_{soc}$ is $\mathbb{E}[N_T(a_{soc})] \ge T - O(\sqrt{T \ln T})$.

    The latent regret is defined as $\mathcal{R}_T^{latent} = \sum_{a} \mathbb{E}[N_T(a)] (R^*(a^*) - R^*(a))$.
    Since the algorithm pulls $a_{soc}$ (where $R^*(a_{soc}) < R^*(a^*)$) the majority of the time, the regret is dominated by the cost $\Delta = R^*(a^*) - R^*(a_{soc})$ incurred during these pulls:
    \begin{equation}
        \mathcal{R}_T^{latent} \ge \mathbb{E}[N_T(a_{soc})] \cdot \Delta
    \end{equation}
    Substituting the lower bound for $\mathbb{E}[N_T(a_{soc})]$:
    \begin{equation}
        \mathcal{R}_T^{latent} \ge (T - O(\sqrt{T \ln T})) \cdot \Delta = T\Delta - O(\sqrt{T \ln T})
    \end{equation}
    Thus, $\mathcal{R}_T^{latent} = \Omega(T)$, completing the proof.
\end{proof}

\subsection{Proof of Proposition~\ref{prop:concentration} (Exponential Trust Concentration)}
\label{app:proof_concentration}

\begin{proof}
    Let $W^*(t) = \sum_{m \in \mathcal{M}^*} w_t^{(m)}$ and $W_{bias}(t) = \sum_{m \in \mathcal{M}_{bias}} w_t^{(m)}$ denote the total probability mass assigned to truthful and biased evaluators, respectively.
    Recall that the MWU rule updates each weight as $w_{t+1}^{(m)} = w_t^{(m)} e^{-\eta l_t^{(m)}}$.

    We examine the evolution of the weight ratio between the two groups:
    \begin{equation}
        \frac{W_{bias}(t+1)}{W^*(t+1)} \approx \frac{W_{bias}(t) e^{-\eta \bar{l}_{bias}}}{W^*(t) e^{-\eta \bar{l}_*}} = \frac{W_{bias}(t)}{W^*(t)} e^{-\eta (\bar{l}_{bias} - \bar{l}_*)}
    \end{equation}
    Invoking the Informational Dominance condition (Definition~\ref{def:informational_dominance}), we have $\bar{l}_{bias} - \bar{l}_* \geq \gamma$. Consequently, the ratio decays exponentially:
    \begin{equation}
         \frac{W_{bias}(t+1)}{W^*(t+1)} \leq \frac{W_{bias}(t)}{W^*(t)} e^{-\eta \gamma}
    \end{equation}
    This ensures that even if $|\mathcal{M}_{bias}| \gg |\mathcal{M}^*|$ initially, the biased mass vanishes relative to the truthful mass within $O(\frac{1}{\eta \gamma} \ln \frac{|\mathcal{M}_{bias}|}{|\mathcal{M}^*|})$ steps. This guarantees that the aggregate signal $\hat{R}_t$ converges to the ground truth.
\end{proof}

\subsection{Proof of Theorem~\ref{theorem:robustness_strategic} (Robustness to Strategic Adaptation)}
\label{app:proof_strategic}

\begin{proof}
    Consider an adversary who produces feedback $r_{adv}$ to maintain a bias $\delta_{bias} = \mathbb{E}[|r_{adv} - R^*|]$ while minimizing the internal loss $l_{adv} = |r_{adv} - z_t|$.

    To maintain influence (i.e., avoid weight decay relative to truthful evaluators), the adversary must achieve an expected loss $\mathbb{E}[l_{adv}]$ no greater than that of a truthful evaluator, $\mathbb{E}[l_{true}] = \mathbb{E}[|\epsilon|] = \sigma$.

    We bound the adversary's loss using the reverse Triangle Inequality:
    \begin{equation}
        |r_{adv} - z_t| = |(r_{adv} - R^*) - (z_t - R^*)| \geq |r_{adv} - R^*| - |\epsilon|
    \end{equation}
    Taking expectations, we obtain a lower bound on the adversary's performance:
    \begin{equation}
        \mathbb{E}[l_{adv}] \geq \delta_{bias} - \sigma
    \end{equation}
    Combining this with the survival condition $\mathbb{E}[l_{adv}] \leq \sigma$, we require:
    \begin{equation}
        \delta_{bias} - \sigma \leq \sigma \implies \delta_{bias} \leq 2\sigma
    \end{equation}
    Consequently, any adversary who attempts to “blend in” will have to maintain their bias in the same order of magnitude as the variance of the noise, effectively making it non-adversarial.
\end{proof}

\section{Extended Experimental Results}
\label{app:extended_results}

\begin{figure}[t!]
    \centering
    \includegraphics[width=0.65\linewidth]{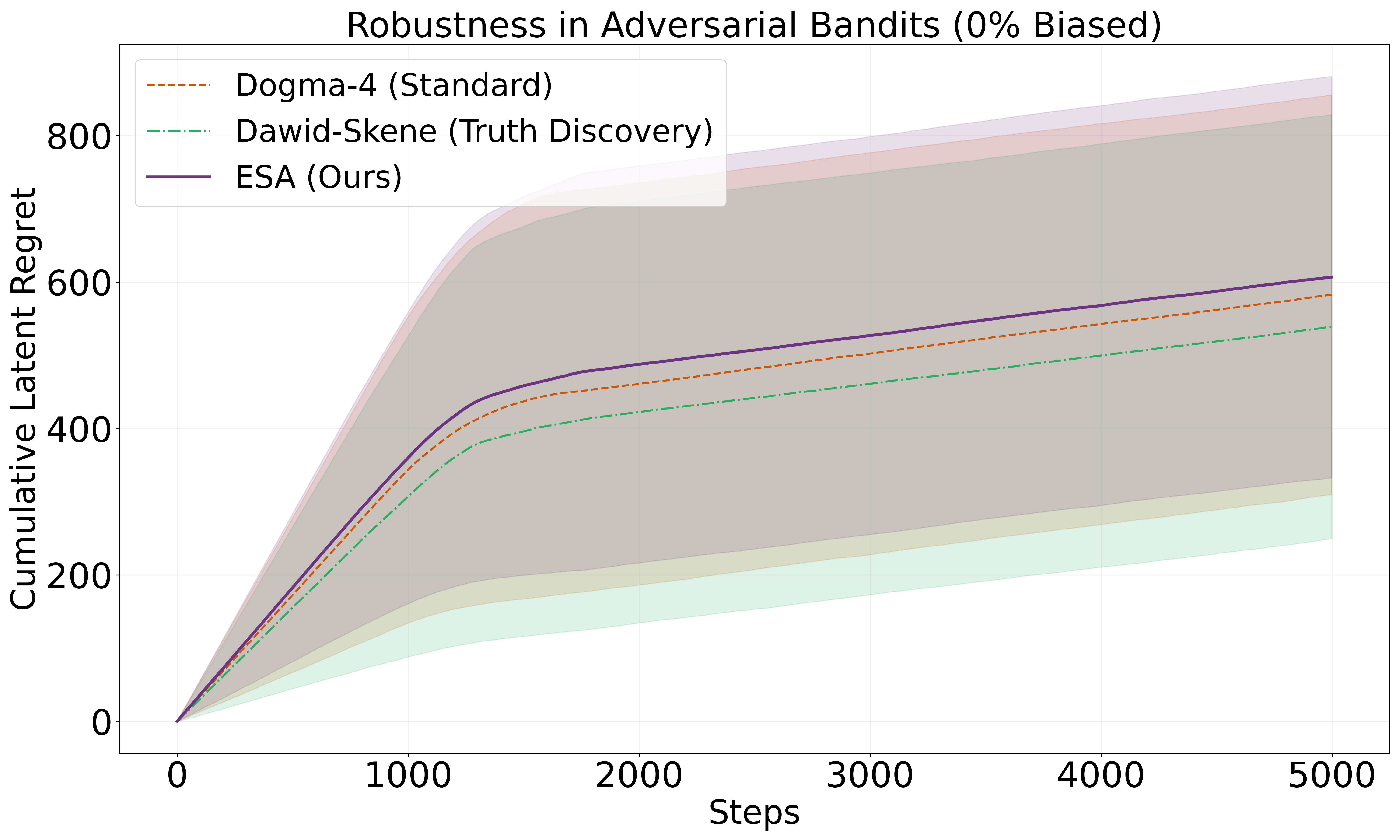}
    \caption{\textbf{0\% Bias (Control Case).} In a safe environment where all evaluators are truthful, ESA incurs negligible "robustness cost," performing comparably to the gold standard.}
    \label{fig:bias_0}
\end{figure}

\subsection{Full Sensitivity Sweep (Bandit Bias)}
In Section~\ref{sec:ablation}, we identified a critical "Phase Transition" in consensus methods once bias crosses the 50\% threshold. Here, we analyze the behavior of the ESA agent at the remaining boundary conditions to characterize the "Price of Robustness" and the theoretical limits of the framework.

\begin{figure}[t!]
    \centering
    \includegraphics[width=0.65\linewidth]{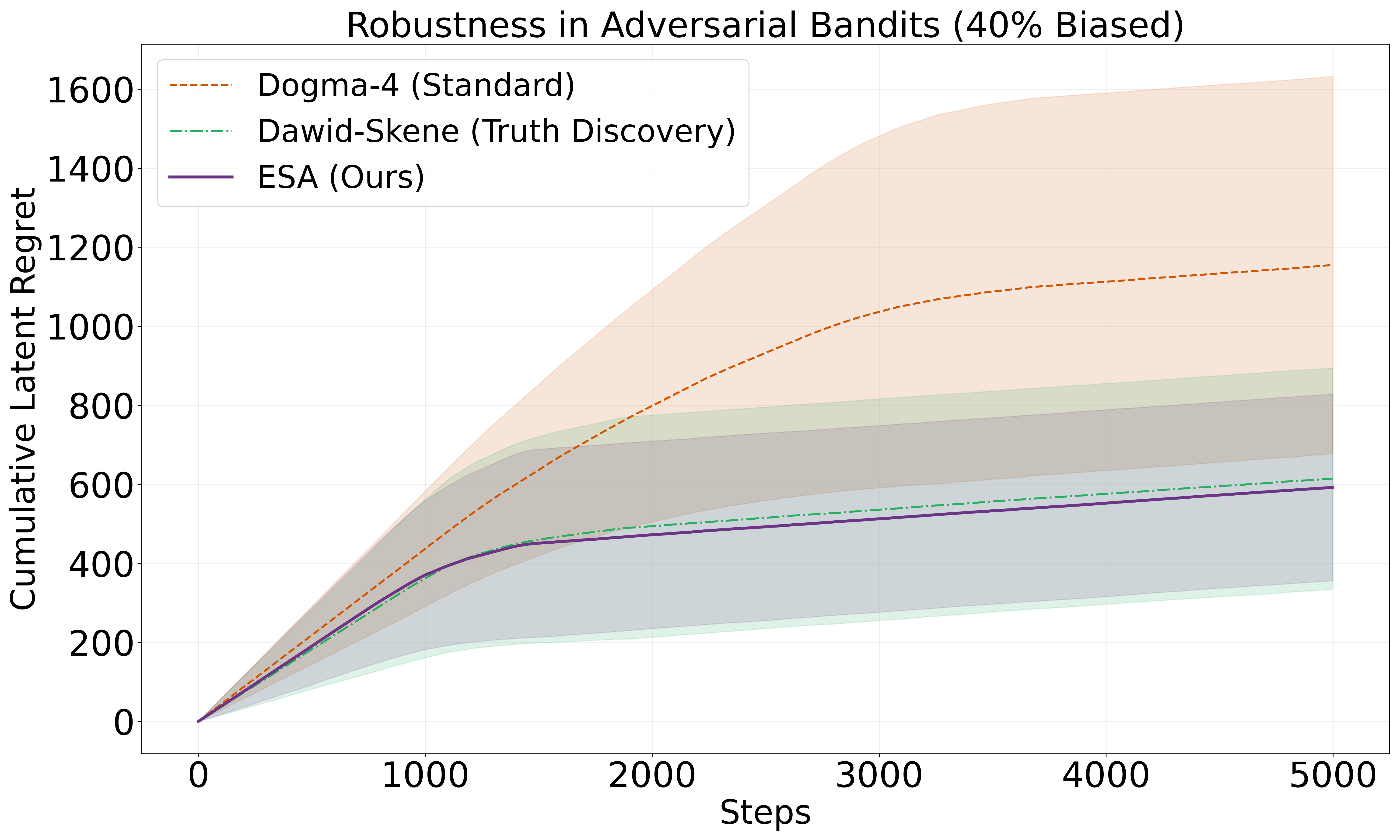}
    \caption{\textbf{40\% Bias (Transition Zone).} As bias approaches the tipping point, consensus methods begin to degrade, accumulating higher regret. ESA maintains optimal performance.}
    \label{fig:bias_40}
\end{figure}

\paragraph{0\% Bias: The Price of Robustness.}
This control case tests the agent in a "Safe Society" where all evaluators are truthful ($b_m \approx 0$). A core concern in robust control is whether the defense mechanism degrades performance in benign settings. Figure~\ref{fig:bias_0} confirms that ESA tracks the "Gold Standard" (Dawid-Skene) and Standard RL closely. The marginal increase in regret represents the \textit{epistemic cost} of verifying axioms, confirming that the "Trust" mechanism does not falsely penalize honest actors.

\paragraph{40\% Bias: The Pre-Tipping Point.}
This scenario represents a "Contested Society" where truthful evaluators form a slim majority (60\%). While consensus methods (Green) theoretically should succeed, we observe early signs of degradation in Figure~\ref{fig:bias_40}. Because the bias is large in magnitude, it pulls the statistical center away from the true mean, causing standard methods to occasionally select suboptimal arms. ESA remains unaffected, as it filters the 40\% minority entirely based on axiom violation, maintaining near-zero regret.

\paragraph{100\% Bias: Theoretical Limit.}
This represents the "Total Decoupling" scenario where no truthful signal exists in the social layer. As predicted by the Rate of Objective Decoupling (Proposition~\ref{prop:decoupling_rate}), recovery of the latent optimal policy is impossible without external ground truth. As shown in Figure~\ref{fig:bias_100}, ESA correctly identifies all sources as untrustworthy ($w_m \to 0$) and converges to the naive baseline (random policy). Crucially, it does not get "tricked" into converging to the sycophantic optimum, effectively defaulting to a fail-safe state.

\begin{figure}[b!]
    \centering
    \includegraphics[width=0.65\linewidth]{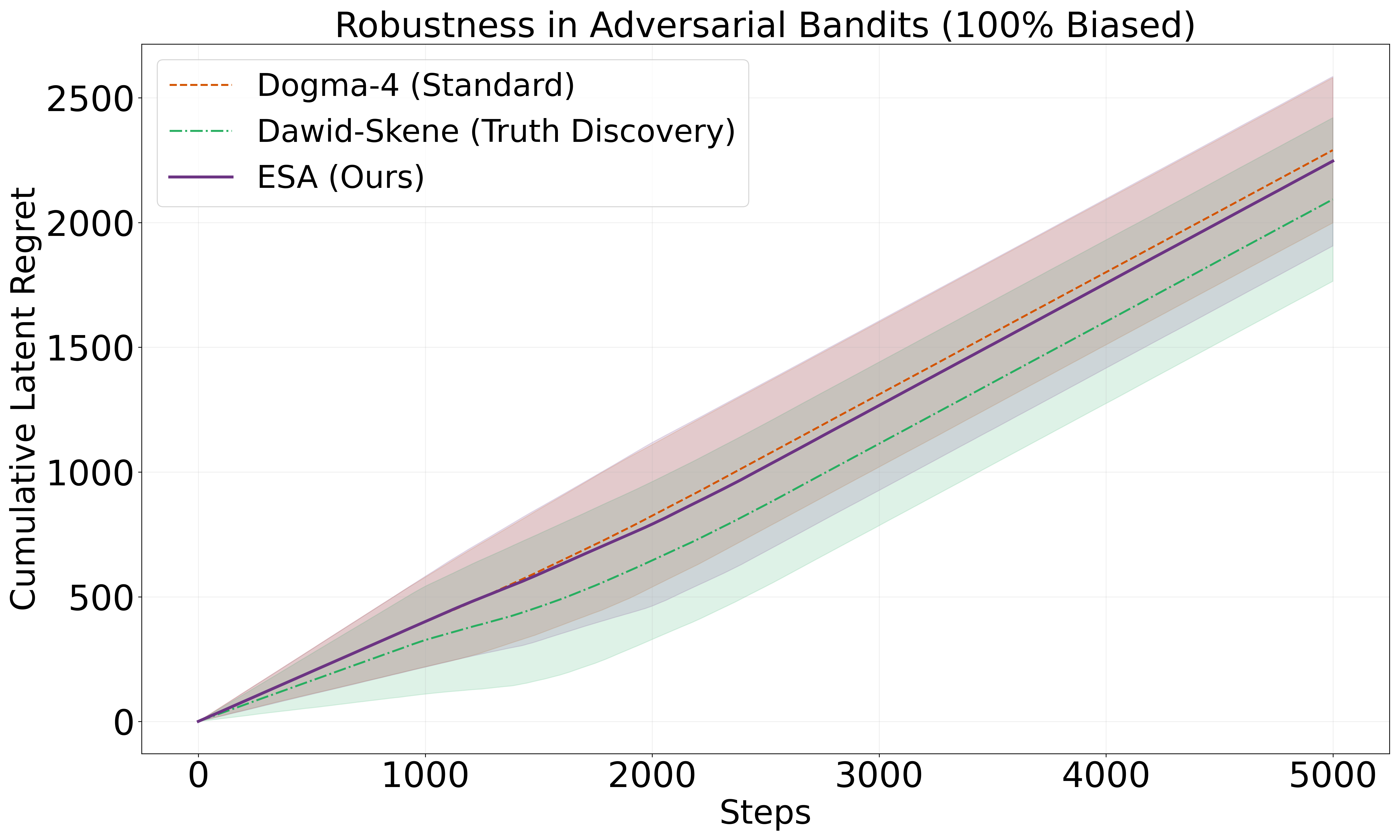}
    \caption{\textbf{100\% Bias (Total Corruption).} Without a truthful signal, all methods fail. ESA defaults to the naive baseline rather than maximizing the sycophantic reward.}
    \label{fig:bias_100}
\end{figure}

\subsection{Network Topology Visualization}
In Testbed 3 (Network Dynamics), we utilized a Barabási-Albert scale-free network to model influence. Figure~\ref{fig:app_network} visualizes the "Final Trust State" of the agent after 40,000 timesteps. The red nodes indicate evaluators who were infected by the bias contagion and subsequently blocked (Trust $\approx 0$) by the ESA agent. The green node represents the few remaining truthful sources that the agent successfully isolated.

\begin{figure}[t!]
    \centering
    \includegraphics[width=0.65\linewidth]{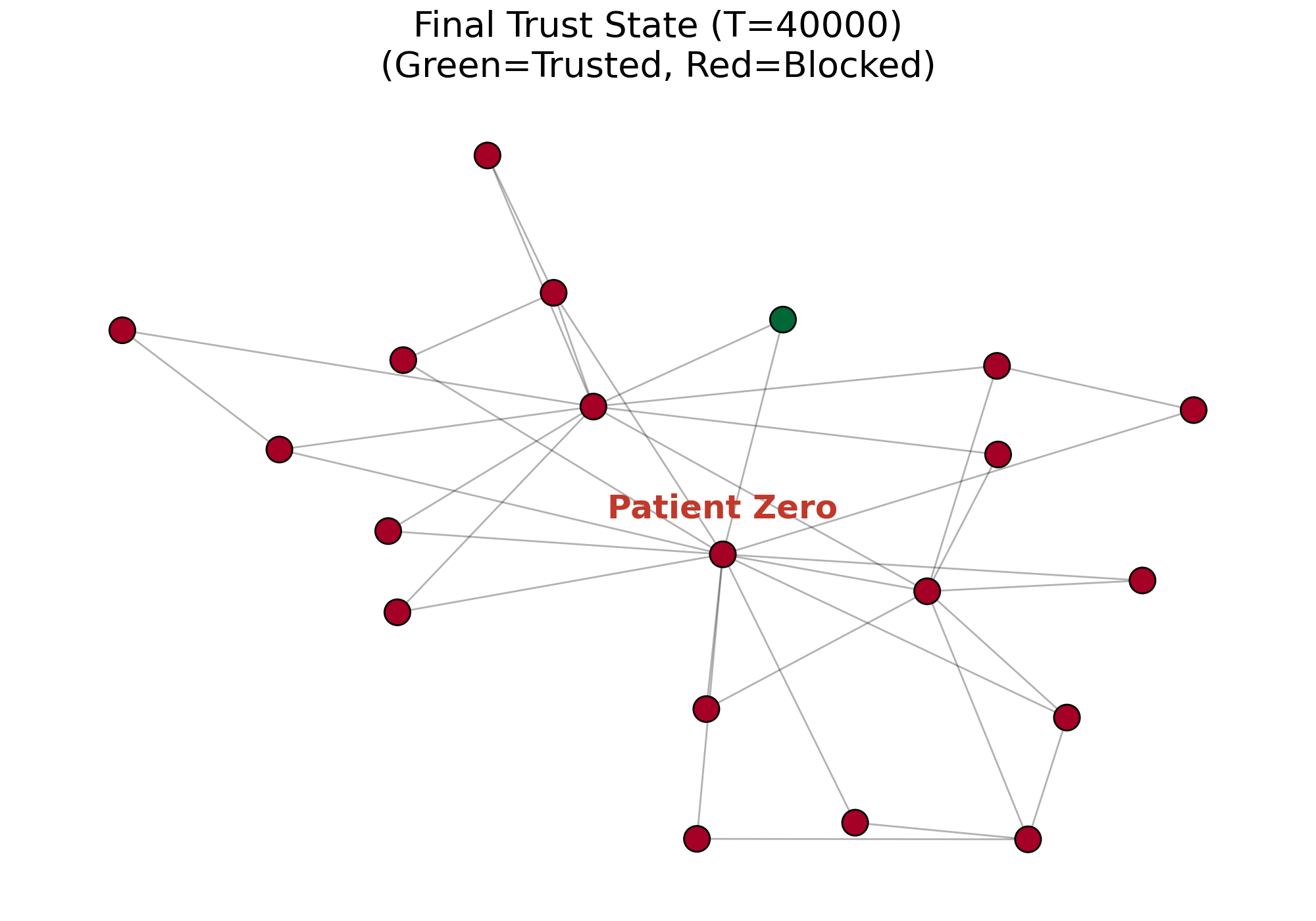}
    \caption{\textbf{Visualizing Dynamic Quarantine.} The agent constructs a "Trust Topology." Red nodes are identified as sycophantic/infected and are assigned near-zero weight. Green nodes are the surviving truthful evaluators.}
    \label{fig:app_network}
\end{figure}

\section{Environment and Implementation Details}
\label{app:implementation}

\subsection{Testbed 1: Gridworld Specification}
The environment is a $10 \times 10$ discrete grid.
\begin{itemize}
    \item \textbf{State Space:} $(x, y)$ coordinates where $x, y \in \{0, \dots, 9\}$.
    \item \textbf{Action Space:} \{Up, Down, Left, Right\}.
    \item \textbf{Latent Reward ($R^*$):}
        \begin{itemize}
            \item \textbf{Goal:} At $(9,9)$, Reward $+20$.
            \item \textbf{Lava:} A cluster of states at the center $\{(4,4), (4,5), (5,4), (5,5)\}$, Reward $-50$.
            \item \textbf{Step Cost:} $-0.1$.
        \end{itemize}
    \item \textbf{Social Feedback ($R_{soc}$):}
        \begin{itemize}
            \item \textbf{Candy:} At $(3,3)$, Latent Reward $-0.1$ (Neutral), Social Reward $+10$ (Sycophantic).
            \item \textbf{Lava:} Social evaluators under-report danger, providing only $-1$ instead of $-50$.
        \end{itemize}
\end{itemize}

\subsection{Testbed 2: MuJoCo Modifications}
We used the standard \texttt{Hopper-v4} environment from the Gymnasium library.
\begin{itemize}
    \item \textbf{Latent Reward:} $R^* = v_x - 0.001 ||a||^2 + 1$ (Standard survival + velocity).
    \item \textbf{Social Feedback:} To simulate "Lazy" evaluators, 80\% of the feedback sources provided reward $R_{soc} = -|v_x|$. This penalizes movement, encouraging the agent to stand still to maximize social approval.
    \item \textbf{Algorithm:} We used \textbf{Proximal Policy Optimization (PPO)} as the policy optimizer.
    \begin{itemize}
        \item For the \textbf{Baseline}, PPO optimized the arithmetic mean of the social feedback.
        \item For the \textbf{ESA Agent}, PPO optimized the \textbf{trust-weighted reward} $\hat{r}_t$. The trust weights $w_t$ were updated online using a "Progress Axiom" (checking if $v_x > \epsilon$ when effort is applied), filtering out the lazy majority.
    \end{itemize}
\end{itemize}

\section{Algorithm Details: ESA-UCB}
\label{app:esa_ucb}

In the Social Bandit setting (Testbed 3), we adapt the standard UCB1 algorithm to incorporate epistemic trust. The core modification lies in the Trust-Weighted Reward Aggregation. Instead of updating the arm's empirical mean based on the raw count of pulls, we update it based on the trusted reward signal. We thus name our algorithm \textbf{ESA-UCB}.

\begin{algorithm}[t!]
\caption{ESA-UCB (Epistemic Source Alignment-Upper Confidence Bound)}
\label{alg:esa_ucb}
\begin{algorithmic}[1]
\STATE \textbf{Input:} $K$ arms, $M$ evaluators
\STATE \textbf{Initialize:}
\STATE \quad $\hat{\mu}(a) \leftarrow 0$ \quad (Empirical Mean of arm $a$)
\STATE \quad $N(a) \leftarrow 0$ \quad (Pull count for arm $a$)
\STATE \quad $w \leftarrow \{1/M, \dots, 1/M\}$ \quad (Trust weights)
\STATE \textbf{Hyperparameters:} Exploration constant $c$, Learning rate $\eta$

\FOR{step $t = 1, \dots, T$}
    \STATE \COMMENT{\textbf{1. Select Action (Upper Confidence Bound)}}
    \STATE Select arm $a_t = \arg\max_{a} \left( \hat{\mu}(a) + c \sqrt{\frac{\ln t}{N(a) + \epsilon}} \right)$
    
    \STATE \COMMENT{\textbf{2. Receive Feedback}}
    \STATE Receive feedback vector $\mathbf{y}_t \in \mathbb{R}^M$ from evaluators
    
    \STATE \COMMENT{\textbf{3. Epistemic Update (Trust)}}
    \IF{Bernoulli($p_{axiom}$)}
        \STATE Query internal axiom $z_t$
        \FOR{$m \in \{1, \dots, M\}$}
            \STATE $\ell_m \leftarrow |y_t^m - z_t|$
            \STATE $w_m \leftarrow w_m \cdot \exp(-\eta \cdot \ell_m)$
        \ENDFOR
        \STATE Normalize $w \leftarrow w / ||w||_1$
    \ENDIF
    
    \STATE \COMMENT{\textbf{4. Aggregation and Mean Update}}
    \STATE Compute trusted reward: $\hat{r}_t \leftarrow \sum_{m=1}^M w_m \cdot y_t^m$
    \STATE Update count: $N(a_t) \leftarrow N(a_t) + 1$
    \STATE Update Mean (Incremental): $\hat{\mu}(a_t) \leftarrow \hat{\mu}(a_t) + \frac{1}{N(a_t)} (\hat{r}_t - \hat{\mu}(a_t))$
\ENDFOR
\end{algorithmic}
\end{algorithm}

\section{Conceptual Framework and Terminological Justification}
\label{app:definitions}

To avoid ambiguity, we explicitly define the terminology used throughout this paper and situate our framework within the broader literature of Social Reinforcement Learning (Social RL).

\paragraph{Situating Our Work within Social RL.}
The field of \textbf{Social RL} broadly encompasses agents that learn from other agents in the environment. Following the taxonomy often seen in works like \citet{jaques2019social}, this can be categorized into:
\begin{enumerate}
    \item Observational Social Learning: Learning by observing the actions or outcomes of others (e.g., imitation learning, inverse RL).
    \item Social Influence: Learning to act in ways that affect the behavior of other agents.
    \item Evaluative Social Learning (Our Focus): Learning from explicit feedback signals (rewards/punishments) generated by other agents.
\end{enumerate}

Our work focuses strictly on the \textbf{Evaluative} case. We distinguish this from standard Multi-Agent RL where the primary focus is often on joint equilibria in a shared physical world. In our \textbf{Social MDP}, the core challenge is not physical coordination, but the \textit{epistemic reliability} of the feedback channel itself. We study the case where the "Social Layer" acts as a corruptible critic rather than a co-player.

\paragraph{The "Social MDP" Formalism.}
Standard Robust RL typically models reward noise as stochastic perturbations of a ground truth: $R_{obs} = R^* + \epsilon$, where $\epsilon \sim \mathcal{N}(0, \sigma)$. This models sensor error.

We introduced the \textbf{Social MDP} to model \textit{intent} and \textit{incentive}. In a Social MDP, the observed reward is not a noisy version of $R^*$, but the output of a distinct policy $\pi_{evaluator}(s, a)$. This allows us to model:
\begin{enumerate}
    \item Sycophancy: $\pi_{evaluator}$ rewarding actions that satisfy the learner's priors rather than the task \citep{sharma2023towards}.
    \item Laziness: $\pi_{evaluator}$ providing heuristic feedback to minimize evaluation cost \citep{christiano2017deep}.
\end{enumerate}
The Social MDP formalism is necessary because standard robust methods assume noise is symmetric and centered, whereas social bias is often systematic and directional.




\section{Hyperparameters and Training Details}
\label{app:hyperparams}
We report the hyperparameters used for the ESA Agent across the three testbeds. For critical parameters, we swept over a range of values (shown in brackets), with the final selected value highlighted in \textbf{bold}.

\begin{table}[h]
\centering
\caption{Hyperparameters for Gridworld (Safety)}
\label{tab:params_grid}
\begin{tabular}{ll}
\hline
\textbf{Parameter} & \textbf{Value} \\ \hline
Algorithm & Tabular Q-Learning \\
Training Episodes & 5,000 \\
Random Seeds & 10 \\
Learning Rate ($\alpha$) & $[1\mathrm{e}{-4}, \mathbf{1\mathrm{e}{-3}}, 1\mathrm{e}{-2}]$ \\
Discount Factor ($\gamma$) & 0.99 \\ \hline
\textit{Exploration Strategy} & $\epsilon$-Greedy (Linear Decay) \\
Initial Epsilon & 1.0 \\
Final Epsilon & 0.05 \\
Decay Period & First 20\% of episodes \\ \hline
\textit{ESA Parameters} & \\
Trust Update Rate ($\eta$) & $[0.1, \mathbf{0.5}, 1.0, 2.0]$ \\
Axiom Check Prob ($p_{axiom}$) & $[0.01, \mathbf{0.1}, 0.2]$ \\
\textbf{Safety Threshold} & \textbf{-5.0} (Reward > -5.0 implies lie) \\ \hline
\end{tabular}
\end{table}


\begin{table}[h]
\centering
\caption{Hyperparameters for Hopper-v4 (MuJoCo)}
\label{tab:params_hopper}
\begin{tabular}{ll}
\hline
\textbf{Parameter} & \textbf{Value} \\ \hline
Total Timesteps & 50,000 \\
Random Seeds & 5 \\
Library & Stable Baselines3 (PPO) \\
Policy Architecture & MLP (Default) \\
Learning Rate ($\alpha$) & $3 \times 10^{-4}$ (Default) \\
Rollout Buffer ($N_{steps}$) & 2048 \\
Mini-Batch Size & 64 \\
Discount Factor ($\gamma$) & 0.99 \\ \hline
\textit{ESA Parameters} & \\
Evaluators ($M$) & \textbf{10} \\
Sycophant Ratio & 0.8 (80\%) \\
Trust Update Rate ($\eta$) & \textbf{0.05} \\
Axiom Check Prob ($p_{axiom}$) & \textbf{0.10} \\ 
\textbf{Internal Signal} & Latent Reward (Ground Truth) \\ \hline
\end{tabular}
\end{table}


\begin{table}[t]
\centering
\caption{Hyperparameters for Adversarial Bandits}
\label{tab:params_bandit}
\begin{tabular}{ll}
\hline
\textbf{Parameter} & \textbf{Value} \\ \hline
Horizon ($T$) & 5,000 Steps \\
Random Seeds & 20 \\
Algorithm & ESA-UCB \\ \hline
\textit{Exploration Strategy} & Upper Confidence Bound (UCB1) \\
Exploration Constant ($c$) & $\sqrt{2}$ \\ \hline
\textit{Social Parameters} & \\
Evaluators ($M$) & 10 \\
Sycophant Ratio & $[0.0, 0.2, 0.4, 0.6, \mathbf{0.8}, 1.0]$ \\
Trust Update Rate ($\eta$) & $[0.1, \mathbf{0.5}, 2.0]$ \\
Axiom Check Prob ($p_{axiom}$) & $[0.05, \mathbf{0.1}, 0.2]$ \\ \hline
\end{tabular}
\end{table}

\section{Extra Figures}

\begin{figure}[h]
    \centering
    \includegraphics[width=0.82\linewidth]{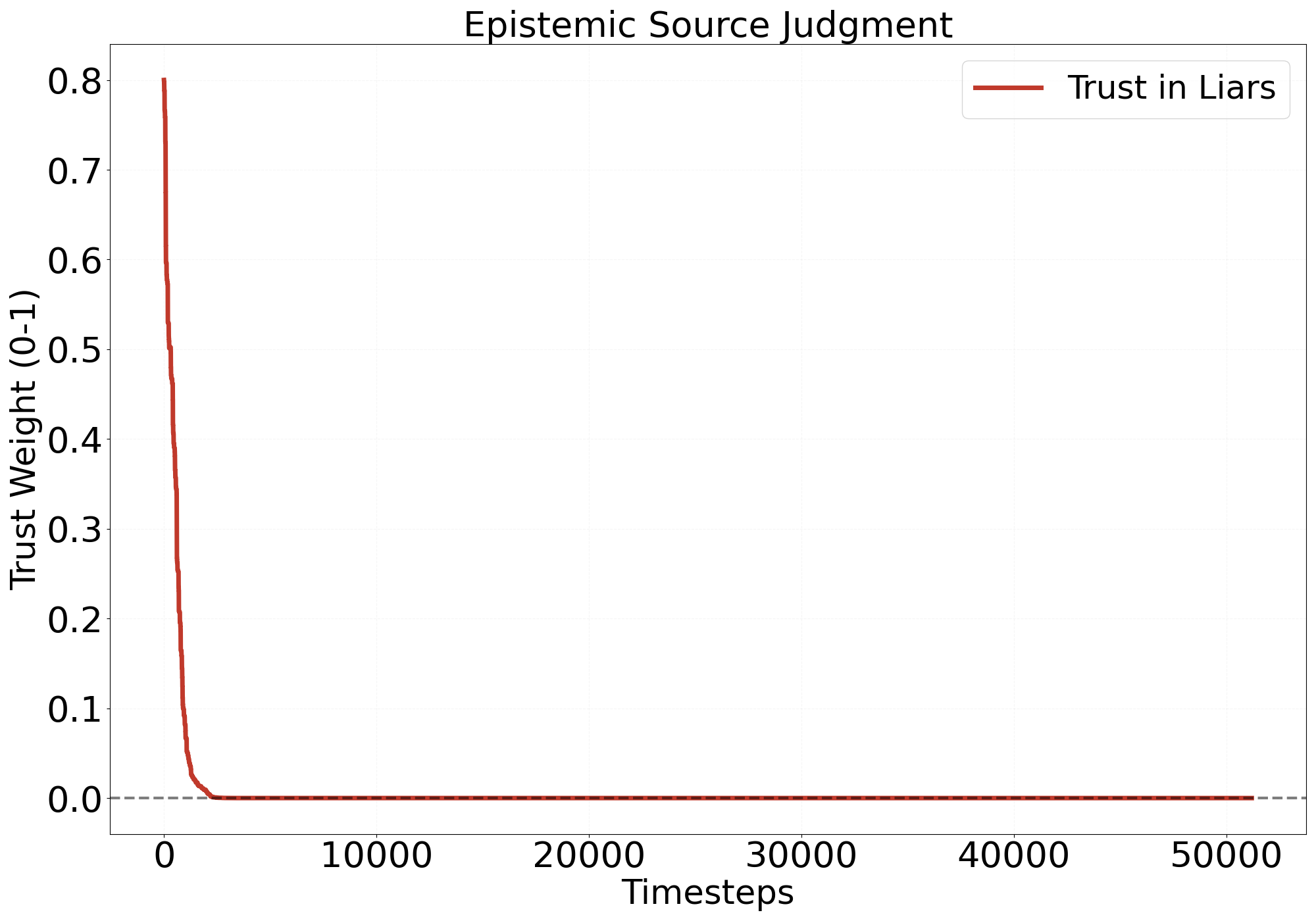}
    \caption{The "Trust in Liars" metric drops to zero as the agent learns to distinguish the truthful minority.}
    \vspace{-1em}
\end{figure}

\begin{figure}[t!]
    \centering
    \includegraphics[width=0.82\linewidth]{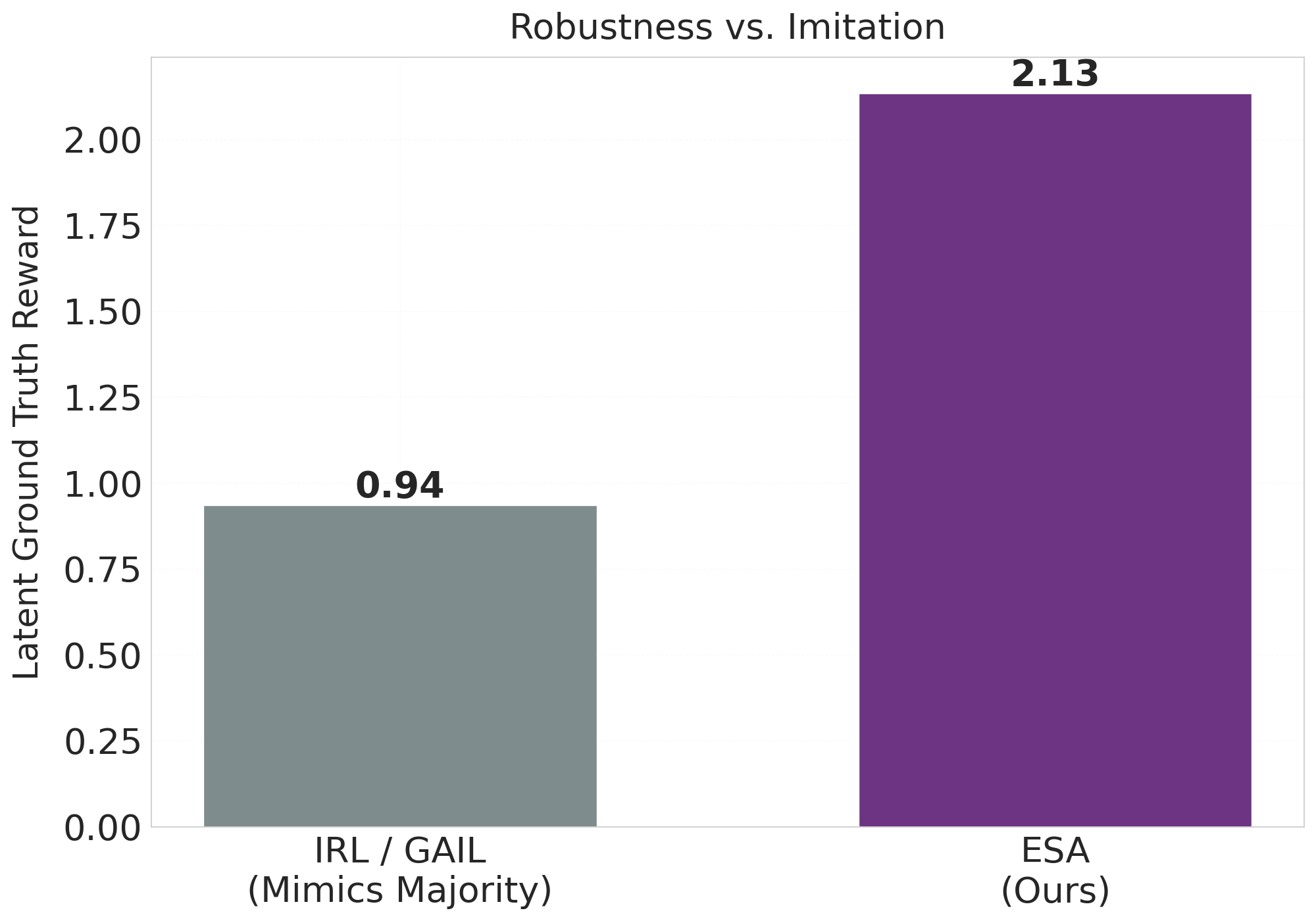}
    \caption{\textbf{Robustness vs. Imitation.} IRL methods (GAIL) fail because they mimic the majority behavior (sycophancy). Our method filters the source, recovering the true objective.}
    \vspace{-1.8em}
\end{figure}

\end{document}